\documentclass{article}

\usepackage{microtype}
\usepackage{graphicx}
\usepackage{adjustbox}
\usepackage{booktabs}
\usepackage[table]{xcolor}

\usepackage{hyperref}

\usepackage[preprint]{icml2026}

\usepackage{amsmath,amssymb}
\usepackage{amsthm}
\usepackage[nameinlink,noabbrev]{cleveref}

\icmltitlerunning{Diagnosing Dynamic Instability in LLM Reasoning}

\newcommand{\It}{I_t}
\newcommand{\temp}{\tau}
\newcommand{\taumix}{\tau_{\mathrm{mix}}}

\newcommand{\QwenTwoFiveFull}{\texorpdfstring{Qwen2.5-1.5B-Instruct}{Qwen2.5-1.5B-Instruct}}

\newcommand{\LlamaFull}{\texorpdfstring{Llama-3.2-1B-Instruct}{Llama-3.2-1B-Instruct}}

\newcommand{\LlamaThreeBFull}{\texorpdfstring{Llama-3.2-3B-Instruct}{Llama-3.2-3B-Instruct}}

\newcommand{\LlamaEightBFull}{\texorpdfstring{Llama-3.1-8B-Instruct}{Llama-3.1-8B-Instruct}}

\newcommand{\QwenHalfBInst}{\texorpdfstring{Qwen2.5-0.5B-Inst}{Qwen2.5-0.5B-Inst}}
\newcommand{\QwenTwoFiveInst}{\texorpdfstring{Qwen2.5-1.5B-Inst}{Qwen2.5-1.5B-Inst}}
\newcommand{\QwenTwoFiveSevenBInst}{\texorpdfstring{Qwen2.5-7B-Inst}{Qwen2.5-7B-Inst}}
\newcommand{\LlamaOneBInst}{\texorpdfstring{Llama-3.2-1B-Inst}{Llama-3.2-1B-Inst}}
\newcommand{\LlamaThreeBInst}{\texorpdfstring{Llama-3.2-3B-Inst}{Llama-3.2-3B-Inst}}
\newcommand{\LlamaEightBInst}{\texorpdfstring{Llama-3.1-8B-Inst}{Llama-3.1-8B-Inst}}

\newtheorem{definition}{Definition}
\newtheorem{lemma}{Lemma}
\newtheorem{proposition}{Proposition}
\newtheorem{theorem}{Theorem}

\begin{document}

\twocolumn[
  \icmltitle{``I May Not Have Articulated Myself Clearly'': Diagnosing Dynamic Instability in LLM Reasoning at Inference Time}

  \begin{icmlauthorlist}
    \icmlauthor{Jinkun Chen}{dal}
    \icmlauthor{Fengxiang Cheng}{uva,tsinghua}
    \icmlauthor{Sijia Han}{meta}
    \icmlauthor{Vlado Keselj}{dal}
  \end{icmlauthorlist}

  \icmlaffiliation{dal}{Dalhousie University}
  \icmlaffiliation{uva}{University of Amsterdam}
  \icmlaffiliation{tsinghua}{Tsinghua University}
  \icmlaffiliation{meta}{Meta}

  \icmlcorrespondingauthor{Jinkun Chen}{jinkun.chen@dal.ca}

  \icmlkeywords{Large language models, reasoning, inference-time diagnostics}

  \vskip 0.3in
]

\printAffiliationsAndNotice{}  

\begin{abstract}
Reasoning failures in large language models (LLMs) are typically measured only at the end of a
generation, yet many failures manifest as a process-level breakdown: the model ``loses the thread''
mid-reasoning. We study whether such breakdowns are detectable from inference-time observables
available in standard APIs (token log probabilities), without any training or fine-tuning.
We define a simple instability signal that combines consecutive-step distributional shift (JSD) and
uncertainty (entropy), summarize each trace by its peak instability strength, and show that this
signal reliably predicts failure.
Across GSM8K and HotpotQA, instability strength predicts wrong answers with above-chance AUC and
yields monotonic bucket-level accuracy decline at scale across model sizes.
Crucially, we show that instability is not uniformly harmful: early instability can reflect
subsequent stabilization and a correct final answer (\emph{corrective instability}), whereas late
instability is more often followed by failure (\emph{destructive instability}), even at comparable
peak magnitudes, indicating that recoverability depends not only on how strongly the distribution
changes but also on when such changes occur relative to the remaining decoding horizon.
The method is model-agnostic, training-free, and reproducible, and is presented as a diagnostic
lens rather than a corrective or control mechanism.
\end{abstract}

\begin{figure}[t]
\centering
\includegraphics[width=\linewidth]{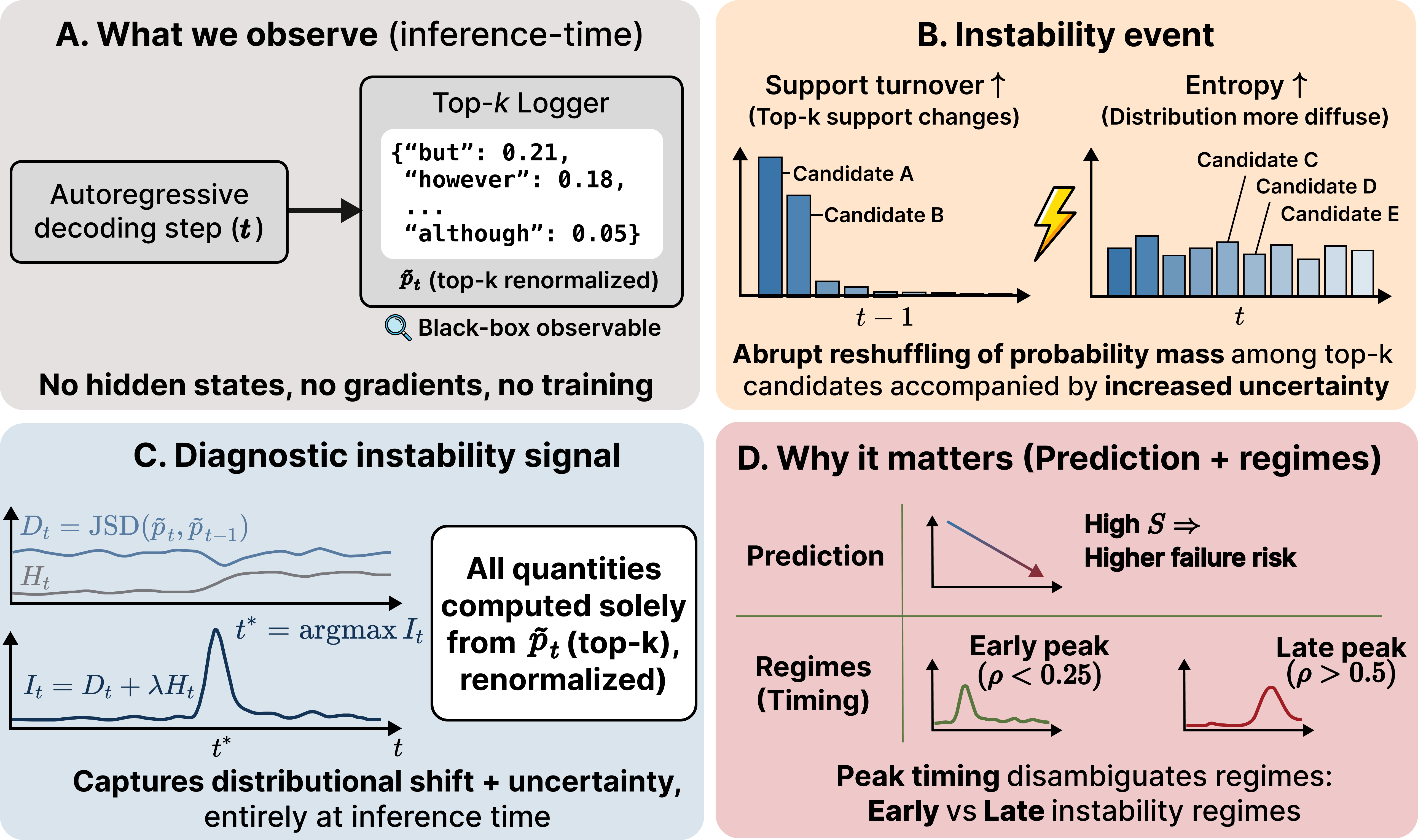}
\caption{Inference-time dynamic instability as a diagnostic signal.
Conceptual overview of an inference-time diagnostic for dynamic instability in large language model reasoning.
\textbf{A}: At each decoding step $t$, we observe only the renormalized top-$k$ next-token distribution $\tilde{p}_t$, which is accessible as a black-box inference-time signal; no hidden states, gradients, or training access are required.
\textbf{B}: An instability event is characterized by an abrupt reshuffling of probability mass among high-probability candidates, accompanied by increased uncertainty (higher entropy) between consecutive steps.
\textbf{C}: We define a diagnostic instability signal $\It = D_t + \lambda H_t$, combining distributional shift $D_t = \mathrm{JSD}(\tilde{p}_t,\tilde{p}_{t-1})$ and uncertainty $H_t$. All quantities are computed solely from $\tilde{p}_t$ at inference time, with overall instability strength summarized as $S = \max_t \It$.
\textbf{D}: Instability strength $S$ is associated with increased failure risk, while the peak location (relative position $\rho$) provides complementary information, separating earlier (more recoverable/corrective) from later (less recoverable/destructive) instability episodes.}%
\label{fig:instability_overview}
\end{figure}

\section{Introduction}
LLM reasoning failures are typically evaluated as end-point errors, yet the generation process
is a temporal dynamical system whose internal evolution can become unstable. We study whether
inference-time signals can distinguish a class of failures driven by dynamic instability, rather
than by static knowledge limitations. Our goal is diagnostic rather than corrective: identify
process-level indicators that reliably predict failure, without modifying the model.
\paragraph{Motivation.}
In many applications, the relevant question is not only \emph{whether} an answer is correct, but
\emph{when} the generation begins to break down and whether it is likely to recover.
Final-answer accuracy is retrospective, and static confidence proxies can miss intra-trajectory
transitions; conversely, multi-sample consistency methods require multiple runs.
We therefore ask a minimal inference-time question: can a \emph{single} decoding trace reveal an
observable signature of unstable reasoning using only token probabilities?
\Cref{fig:instability_overview} provides a conceptual overview of our inference-time
diagnostic, illustrating what is observable at decoding time, how instability manifests in token
distributions, and why instability strength and timing are predictive of reasoning failure.
Unlike accuracy-oriented interventions, our analysis does not aim to improve correctness, but to
characterize when and how reasoning processes become unstable.
We focus on detecting and characterizing instability, not on correcting or controlling it.

To avoid confusion with related notions, our instability signal is not a confidence score, a
calibration method, a hallucination detector, or a verifier. It targets trajectory-level
dynamical behavior rather than outcome-level uncertainty: two traces with similar average
entropy can exhibit very different instability profiles, and only the latter reflects temporal
disruption during reasoning.

Why can the next-token distribution represent reasoning stability? Autoregressive decoding induces
a closed-loop state evolution in which the internal state updates based on previously generated
tokens, yet the model's token distribution is the only black-box observable of this evolving
state. Small perturbations may therefore be amplified over time through this feedback loop.
Formally, letting $h_t$ denote the internal state and $x_t$ the emitted token,
\begin{equation}
  h_{t+1}=f(h_t, x_t),\qquad p_t=g(h_t),
\end{equation}
where $p_t$ is the full next-token distribution; in black-box settings we observe $\tilde{p}_t$,
its top-$k$ truncation and renormalization. We therefore study \emph{intra-trajectory temporal
instability}: abrupt changes in $p_t$ (or $\tilde{p}_t$ when truncated) over time \emph{within a single
trace}. This differs from self-consistency and other sampling-based methods, which analyze
\emph{inter-trajectory} variability across multiple samples. It also differs from static uncertainty:
we do not treat entropy as a standalone confidence score, but combine entropy with
consecutive-step distributional shifts to quantify localized regime shifts during a single
generation.

Prior work shows that LLMs often fail to identify their own reasoning errors without explicit
localization, motivating process-level diagnostics of reasoning traces~\cite{tyen2023reasoning}.
Recent methods propose temporal signals for reasoning process error identification and
visualization of reasoning paths, while uncertainty analyses highlight the limits of explanation
faithfulness~\cite{guo2025temporal,li2025reasongraph,da2025uncertainty}. We complement these
directions with an inference-time instability signal defined directly on token distributions.
\Cref{fig:timing_regimes} previews our central observation: instability strength alone is
insufficient; its timing determines whether reasoning recovers or collapses.

\begin{figure}[t]
\centering
\includegraphics[width=\linewidth]{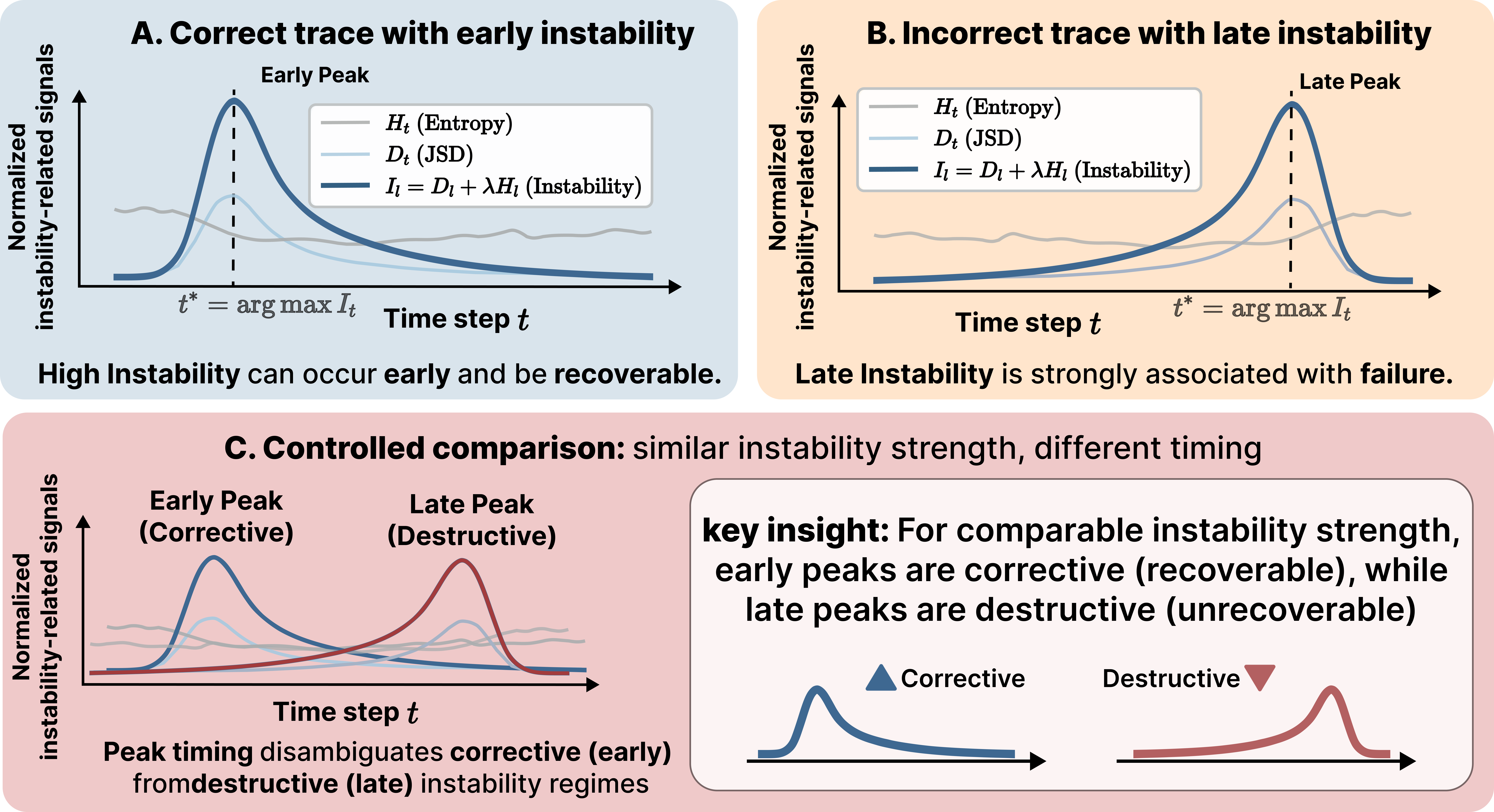}
\caption{Timing-dependent regimes of inference-time instability.
Representative decoding traces illustrating how the \emph{timing} of instability, rather than its
magnitude alone, is associated with different reasoning outcomes.
\textbf{A}: A representative trace with an early instability peak (\textbf{recoverable / corrective}). Although the
instability signal $\It$ reaches a high value, it occurs sufficiently early in decoding for the model
to recover, resulting in a correct final answer.
\textbf{B}: A representative trace with a late instability peak (\textbf{unrecoverable / destructive}). Despite
comparable instability strength, the peak occurs near the end of decoding, leaving limited
opportunity for recovery and resulting in an incorrect final answer.
\textbf{C}: A controlled comparison with instability signals normalized to comparable peak strength,
demonstrating that magnitude alone is insufficient: similar instability strength can correspond to
different outcomes depending on peak timing. All curves show representative (not averaged) decoding
trajectories, with instability defined as $\It = D_t + \lambda H_t$.}%
\label{fig:timing_regimes}
\end{figure}

Our contributions are:
\begin{enumerate}
  \item \textbf{Inference-time instability signal.}
  We define a training-free instability signal $\It = D_t + \lambda H_t$ that combines
  consecutive-step distributional shift and uncertainty, computed solely from token
  log probabilities.
  \item \textbf{Single-trace failure predictability.}
  Instability strength predicts reasoning failure from a \emph{single} decoding trace,
  achieving AUC 0.66--0.74 on GSM8K and remaining predictive at scale on the full GSM8K test
  set and the HotpotQA distractor validation split.
  \item \textbf{Timing-dependent diagnostic refinement.}
  We distinguish \emph{destructive} from \emph{corrective} instability via peak timing,
  showing that instability magnitude alone is insufficient and that recoverability depends
  on when instability occurs.
\end{enumerate}


\paragraph{Scope.}
We test a single claim: instability strength increases with failure rate under both deterministic and stochastic decoding.
We prioritize cross-family reproducibility (Llama/Qwen) across matched size scales (0.5B–8B) for diagnosis rather than exhaustive benchmarking.

\section{Related Work}
Prior work related to our diagnostic can be broadly grouped into four threads:
(i) monitoring and error localization of reasoning trajectories,
(ii) prompting, sampling, and lightweight adaptation for improving reasoning,
(iii) uncertainty and confidence estimation for LLM outputs,
and (iv) adversarial and disruptive perturbations that probe sensitivity of decoding dynamics.
In contrast, we study \emph{intra-trajectory} temporal instability from black-box next-token distributions,
without retraining, additional models, or multi-sample aggregation.

\paragraph{Process monitoring and error localization.}
Reasoning error detection and monitoring have been explored through explicit error localization
and temporal consistency signals~\cite{tyen2023reasoning,guo2025temporal}. Visualization and
structural analysis of reasoning paths provide complementary tools for interpreting failures
and diagnosing process breakdowns~\cite{li2025reasongraph,da2025uncertainty}. Process supervision
and verifier-style training emphasize stepwise correctness but require additional supervision or
fine-tuning~\cite{lightman2023verify}; recent work on process reward models (PRMs) explores
data-efficient stepwise verification via chain-of-thought verifiers~\cite{chen2025thinkprm} and
studies best practices for PRM development~\cite{wang2025lessons_prm}.
Relatedly, \citet{yu2025confidence_over_time} propose viewing confidence as a temporal signal
and use Signal Temporal Logic (STL) to mine temporal patterns that distinguish correct from
incorrect LLM reasoning traces. While their method requires multi-sample generation and STL
mining, our approach operates on a single trace and defines instability directly from
consecutive-step distributional changes.
Our approach is fully inference-time and training-free, and uses only token-distribution observables
available during decoding.

\paragraph{Prompting, sampling, and lightweight adaptation.}
Chain-of-thought prompting can elicit multi-step reasoning~\cite{wei2022cot}, and sampling-based
aggregation such as self-consistency improves accuracy by marginalizing over multiple traces~%
\cite{wang2022selfconsistency}. ReAct-style prompting interleaves reasoning and acting, but can
be brittle in sequential settings when intermediate steps drift~\cite{yao2022react}. Beyond
prompting, parameter-efficient adaptation methods such as prefix-tuning~\cite{li2021prefixtuning},
prompt tuning~\cite{lester2021prompttuning}, and LoRA~\cite{hu2021lora} modulate model behavior
with minimal parameter updates; these methods highlight that small changes in prompt/context can
induce large changes in token-level dynamics. Our instability signal is complementary: rather
than changing the model, we quantify distributional drift and uncertainty along the decoding
trajectory.
Related work also studies prompt sensitivity directly by measuring and manipulating prompt
influence on model outputs~\cite{feng2024promptinfluence}, which supports the broader view that
token-level dynamics can be highly responsive to small input perturbations.
In short, self-consistency and related approaches are, by design, \emph{inter-trajectory} techniques,
whereas we diagnose \emph{intra-trajectory} temporal stability for a single run.

\paragraph{Uncertainty and confidence.}
Black-box self-checking approaches target hallucination detection via consistency signals~%
\cite{manakul2023selfcheckgpt}. Confidence calibration remains an active area, with
multicalibration methods proposed to improve reliability of confidence scores~%
\cite{detommaso2024multicalibration}. Prior work also studies the extent to which model
probabilities reflect epistemic uncertainty and correlate with correctness~\cite{kadavath2022know}.
From an information-theoretic perspective, the next-token entropy corresponds to the conditional
entropy of the next symbol given the prefix in an autoregressive process~\cite{shannon1948mathematical,cover2006elements}.
Sequential inference can be viewed as repeated belief updates that ideally stabilize as evidence
accumulates~\cite{mackay2003information}. Prior work leverages predictive uncertainty (e.g., entropy)
as a mostly static correctness proxy~\cite{guo2017calibration,malinin2018prior_networks}; we extend
this view by analyzing the \emph{temporal evolution} of uncertainty and distributional change along a
single reasoning trajectory, without claiming that information theory alone proves causal failure.

Semantic-level uncertainty measures such as \emph{semantic entropy} (uncertainty computed at the
level of meaning rather than token sequences) have been proposed as principled hallucination
detectors~\cite{farquhar2024semantic_entropy}. Extensions include semantic entropy probes (SEPs)
that extract semantic uncertainty from hidden states~\cite{kossen2024semantic_entropy_probes},
kernel-based fine-grained uncertainty quantification~\cite{nikitin2024kernel_language_entropy},
efficient Bayesian estimation methods~\cite{becker2025bayesian_semantic_entropy}, and
\emph{semantic energy} operating on logits to capture uncertainty in scenarios where semantic entropy
fails~\cite{jiang2025semantic_energy}. Token-level hallucination detection via variance and
probability signals has also been explored~\cite{chen2024token_probability_hallucination,kumar2025token_level_hallucinations_variance,zhang2025adaptive_token_selection}.
Recent evaluations and frameworks compare uncertainty estimation methods and emphasize remaining
challenges~\cite{bakman2024mars,zhang2025llms_estimate_uncertainty,li2025harmonized_uncertainty}.

These methods predominantly target \emph{static} uncertainty at the output or semantic level,
often requiring multiple samples or hidden-state access. In contrast, our focus is specifically
on \emph{temporal} instability within a single decoding trace: we track how the next-token
distribution changes over time, which captures a distinct failure mode (dynamic regime shifts
during reasoning) that static uncertainty measures may miss.

Finally, work on self-correction reaches mixed conclusions: some studies argue that
self-correction is not an inherent capability and depends on scaffolding~%
\cite{huang2023cannot_self_correct_reasoning_yet,liu2024self_correction_not_innate}, while others
find evidence of intrinsic self-correction under certain settings~\cite{liu2024intrinsic_self_correction}.
Recent work trains models to self-correct via reinforcement learning~\cite{kumar2024selfcorrect_rl},
and studies of internal representations show that LLM hidden states encode information predictive
of chain-of-thought success~\cite{halawi2025cot_representations}.
Our ``corrective vs.\ destructive'' distinction is compatible with this picture: high-instability
episodes may correspond either to beneficial revision or to divergence toward failure.
Our method differs in focus: we do not attempt to produce a calibrated confidence score; instead
we characterize \emph{dynamical} instability via temporal changes in token distributions.

\paragraph{Adversarial and disruptive perturbations.}
Recent work shows that stepwise interventions can deliberately disrupt intermediate reasoning and
induce downstream errors~\cite{peng2024stepwise_reasoning_error_disruption_attack}. This supports
our framing that process-level dynamics, not only final answers, constitute a meaningful object of
study, and motivates robust, inference-time diagnostics of instability.
Unlike intervention work (editing/steering), this paper is diagnosis-only: we define and validate
an instability signal without applying stabilization operators.

\section{Method}
\subsection{Inference-Time Observables}
As illustrated in \cref{fig:instability_overview} (\textbf{A, B, and C}), our instability signal is
computed entirely from logged top-$k$ next-token distributions at inference time, without access
to hidden states or full logits.
At each generation step $t$, the model produces a predictive distribution $p_t$ over tokens.
In our implementation, we only observe top-$k$ log probabilities. Let $\mathcal{V}_t$ be the
top-$k$ token set at step $t$ and define the truncated, renormalized distribution
$\tilde{p}_t$ supported on $\mathcal{V}_t$ by setting $\tilde{p}_t(x)=0$ for $x\notin\mathcal{V}_t$
and renormalizing over $\mathcal{V}_t$.
When computing divergences between consecutive steps, we treat $\tilde{p}_t$ and $\tilde{p}_{t-1}$
as distributions over the union support $\mathcal{V}_t\cup\mathcal{V}_{t-1}$ by zero-padding outside
each step's logged set (equivalently, we compute JSD over the union key set).
Unless otherwise noted, we compute all observables from $\tilde{p}_t$; when full vocabulary logits
are available, we compute the same quantities from $p_t$ without truncation.
Our primary objective is a black-box diagnostic based only on log probabilities; therefore we treat
$\tilde{p}_t$ as the operative observable and use quantities computed from full vocabulary logits
only as a validation control when available.

\smallskip
We compute (using natural logarithms):
\begin{align}
  H_t &= -\sum_x \tilde{p}_t(x)\log \tilde{p}_t(x), \\
  D_t &= \mathrm{JSD}(\tilde{p}_t,\tilde{p}_{t-1}),
\end{align}
where $\mathrm{JSD}(P,Q)=\tfrac{1}{2}\mathrm{KL}(P\|M)+\tfrac{1}{2}\mathrm{KL}(Q\|M)$ and
$M=\tfrac{1}{2}(P+Q)$.
With natural logarithms, $\mathrm{JSD}(P,Q)\in[0,\log 2]$.
In all experiments we compute logarithms only on positive probabilities.
We compute KL terms only where the numerator probability is positive (as in the standard definition
of KL), noting that in JSD the mixture $M=\tfrac{1}{2}(P+Q)$ is strictly positive wherever $P$ or
$Q$ is positive.
For numerical stability, we optionally add a small $\epsilon$ (e.g., $10^{-12}$) inside logarithms;
we verified this does not change results.

\smallskip
\subsection{Instability Signal and Strength}
We define instability per-step:
\begin{equation}
\It = D_t + \lambda H_t,
\end{equation}
and note that both correct and incorrect traces can exhibit high instability
(\cref{fig:timing_regimes}, \textbf{A, B, and C}), motivating a focus on temporal structure beyond
magnitude.
We set $\lambda=1$ as a fixed default (equal-weighting in nats) across all tasks, models, and decoding
settings, and do not tune $\lambda$ to avoid optimistic bias from post-hoc hyperparameter selection.
We report sensitivity via a minimal ablation comparing $\lambda=0$ (JSD only) and $\lambda=1$
(\cref{fig:lambda_ablation}), and additional baselines in Appendix~\ref{sec:appendix-controls}.
For each trace, the instability strength is $S = \max_t \It$ to capture brief spikes.
To control for trace-length confounding (longer traces have more opportunities for spikes), we also
report early-window maxima $S_w = \max_{t \le w} \It$ for pre-specified windows $w \in \{10,20,50,100\}$,
highlighting $S_{50}$ as a fixed-window control (\cref{sec:appendix-early-window}).
\paragraph{Implementation and complexity.}
Given per-step top-$k$ log probabilities, we compute $D_t$ by evaluating JSD on the union support
$\mathcal{V}_t\cup\mathcal{V}_{t-1}$ via zero-padding outside each step's logged set, and compute $H_t$
on the renormalized $\tilde{p}_t$.
This yields an $O(Tk)$ per-trace computation (streamable with $O(k)$ working memory), making the diagnostic
practical at scale.
For determinism, when taking $\arg\max$ over time we break ties by the smallest $t$ (as specified for
$t^\star$ and $t^\star_{50}$).

\smallskip
\subsection{Theoretical Intuition}
\label{sec:theory}
Autoregressive decoding is a closed-loop dynamical system: as the model generates tokens, it updates
an unobserved internal state that determines the next-token distribution. Abrupt step-to-step shifts
in the distribution are therefore natural observable signatures of internal regime changes.
Our instability signal combines realized distributional change ($D_t$ via consecutive-step JSD) with
uncertainty ($H_t$ via entropy) to capture both \emph{route switching} and \emph{decision fragility}.
For space, we defer the full theoretical statements (including the recoverability/timing formalization)
to Appendix~\ref{sec:appendix-theory}.

\paragraph{Recoverability under a finite decoding horizon (intuition).}
We view decoding as a finite-horizon process in which entering a stable continuation regime requires
a minimum amount of time. When instability occurs early, there remains sufficient decoding budget
for the model to re-stabilize and converge to a coherent trajectory. In contrast, when instability
peaks late, the remaining budget becomes insufficient for recovery, even if the magnitude of the
instability is comparable.

This perspective highlights that recoverability depends not only on how strongly the model's
distribution changes, but also on when such changes occur relative to the remaining decoding
horizon; we emphasize that this argument is explanatory rather than a formal proof
(Appendix~\ref{sec:appendix-theory}).

\newcommand{\FullTheoryDetails}{%
\subsection{Theoretical Analysis: Dynamic Instability in Autoregressive Reasoning}
\label{sec:appendix-theory-full}
We give a compact theoretical framing that connects internal trajectory instability to observable
changes in next-token distributions. The results below are intentionally limited in scope: they
do not assert that instability causes failure, only that large observable distributional changes
are consistent with substantial internal state changes, and that timing effects admit a simple
recoverability interpretation. Proofs are provided in \cref{sec:appendix-proofs}.
These statements are not used to estimate hidden-state distances or to claim causal mechanisms.
The goal is not to model Transformers in full detail, but to formalize two core ideas: observable
shifts in next-token distributions are mathematically linked to changes in internal state under
mild local conditions, and the timing of such shifts matters for recoverability under a finite
decoding budget.

Where possible we state explicit constants, but our use of these results is qualitative: none of
the bounds are used to set thresholds or guarantee prediction performance, and trajectory-dependent
constants (e.g., curvature terms) may vary widely or degenerate in certain decoding regimes.

\paragraph{Autoregressive reasoning as a dynamical system.}
We model autoregressive inference as a discrete-time closed-loop system. Let $h_t \in \mathcal{H}
\subseteq \mathbb{R}^d$ denote an unobserved internal state at generation step $t$ and let $x_t$ be
the emitted token. The state evolves as
\begin{equation}
  h_{t+1} = f(h_t, x_t).
\end{equation}
The next-token distribution is obtained from logits $z_t$ via an output map and softmax:
\begin{equation}
  z_t = W h_t + b, \qquad p_t = \mathrm{softmax}(z_t),
\end{equation}
where $W \in \mathbb{R}^{V \times d}$ and $V$ is the vocabulary size. In black-box settings we only
observe $\tilde{p}_t$, a truncated and renormalized top-$k$ approximation of $p_t$, and we therefore
compute observables from $\tilde{p}_t$ throughout.

\paragraph{Assumptions.}
We adopt the following local technical assumptions (standard in local stability analyses) only to
make the bounds below well-defined.
\begin{itemize}
  \item \textbf{(A1) Local smoothness.} The state transition $f(h,x)$ is locally differentiable with
  respect to $h$ along the trajectories considered.
  \item \textbf{(A2) Observable gain of the output map.} Let $\Pi := I - \tfrac{1}{V}\mathbf{1}\mathbf{1}^\top$
  be the projection onto $\mathbf{1}^\perp$ (removing the softmax shift invariance). There exists a
  local constant $\sigma_W > 0$ along the trajectories considered such that
  $\lVert \Pi W u \rVert_2 \ge \sigma_W \lVert u \rVert_2$ for all relevant state differences $u$.
\end{itemize}
These assumptions do not guarantee a global stability property for the model. They only support a
local connection between changes in internal state and changes in observable token distributions.
We emphasize that (A2) is a strong, local identifiability-type condition: it may fail globally and
need not hold uniformly over all directions $u$ in practice, but is sufficient to formalize a
stepwise link between hidden-state changes and projected logit changes.

\paragraph{Definitions: finite-horizon stability and local expansion.}
\begin{definition}[Finite-horizon stability]
Given a trajectory $(h_t)_{t=0}^T$, we say it is stable over horizon $T$ if for any sufficiently
small perturbation $\delta h_0$ there exists $C>0$ such that
\begin{equation}
  \sup_{t \le T} \lVert \delta h_t \rVert_2 \le C \lVert \delta h_0 \rVert_2.
\end{equation}
Otherwise the trajectory is dynamically unstable over horizon $T$.
\end{definition}

\begin{definition}[Local expansion rate]
Let $J_f(h_t,x_t) := \tfrac{\partial f}{\partial h}(h_t,x_t)$ and define
\begin{equation}
  \alpha_t := \log \lVert J_f(h_t,x_t) \rVert_{\mathrm{op}}.
\end{equation}
When $\sum_{t=t_0}^{t_1} \alpha_t$ is large and positive, small perturbations expand rapidly over
the interval, indicating an internally unstable segment.
\end{definition}
These quantities are not directly observable in our black-box setting. We use them only to
formalize what it means for an internal trajectory to be unstable.

\paragraph{From internal instability to observable distributional change.}
\begin{lemma}[Hidden-state change induces projected logit change]
\label{lem:state-logit}
Under (A2), for consecutive steps,
\begin{equation}
  \lVert z_t - z_{t-1} \rVert_{\perp} := \lVert \Pi(z_t - z_{t-1}) \rVert_2
  \ge \sigma_W \lVert h_t - h_{t-1} \rVert_2.
\end{equation}
\end{lemma}

\begin{lemma}[Observable distributional change reflects logit change]
\label{lem:logit-jsd}
Let $p_t=\mathrm{softmax}(z_t)$ and $p_{t-1}=\mathrm{softmax}(z_{t-1})$ be consecutive next-token
distributions with logits $z_t,z_{t-1}\in\mathbb{R}^V$. Let $\Pi$ denote the projection onto
$\mathbf{1}^\perp$. Define the trajectory-dependent curvature constant
\begin{equation}
  \kappa_{t,\mathrm{traj}}
  :=
  \inf_{s\in[0,1]}
  \lambda_{\min}\!\Big(J(p(s))\Big|_{\mathbf{1}^\perp}\Big),
\end{equation}
where $z(s)=z_{t-1}+s(z_t-z_{t-1})$, $p(s)=\mathrm{softmax}(z(s))$, and $J(p)=\mathrm{diag}(p)-pp^\top$
is the softmax Jacobian. This quantity captures the minimum local curvature of the softmax map
along the segment between $z_{t-1}$ and $z_t$. Then
\begin{equation}
  \mathrm{JSD}(p_t,p_{t-1})
  \ge \frac{\kappa_{t,\mathrm{traj}}^2}{8}\, \lVert z_t - z_{t-1} \rVert_{\perp}^2.
\end{equation}
\end{lemma}
\noindent\textbf{Remark.} The constant $\kappa_{t,\mathrm{traj}}$ is local and trajectory-dependent.
When the next-token distribution is nearly deterministic (for example near termination), the
softmax Jacobian becomes nearly singular on $\mathbf{1}^\perp$, so $\kappa_{t,\mathrm{traj}}$ can
approach zero and the bound becomes loose. This is consistent with our goal of diagnosing large
transitions during uncertain and decision-critical segments; we do not claim a uniform lower bound
that holds across all decoding regimes.

\begin{proposition}[Top-$k$ analogue (restricted simplex)]
\label{prop:topk_analogue}
Let $\tilde{p}_t$ and $\tilde{p}_{t-1}$ be the renormalized top-$k$ distributions supported on
$\mathcal{V}_t$ and $\mathcal{V}_{t-1}$. Let $\bar{\mathcal{V}}=\mathcal{V}_t\cup\mathcal{V}_{t-1}$
and view $\tilde{p}_t,\tilde{p}_{t-1}$ as distributions on $\bar{\mathcal{V}}$ by zero-padding outside
each step's support. This embedding leaves each step's logged probabilities unchanged and provides
a shared support for divergence computation.
Since any distribution on the restricted simplex with strictly positive coordinates can be written
as a softmax of finite logits (unique up to an additive constant), the statement of
Lemma~\ref{lem:logit-jsd} applies verbatim on the restricted logit space
$\mathbb{R}^{|\bar{\mathcal{V}}|}$, yielding
\begin{equation}
  \mathrm{JSD}(\tilde{p}_t,\tilde{p}_{t-1})
  \ \ge\ \frac{\tilde{\kappa}_{t,\mathrm{traj}}^2}{8}\,
  \big\lVert \tilde{\Pi}_{\bar{\mathcal{V}}}(\tilde{z}_t-\tilde{z}_{t-1})\big\rVert_2^2,
\end{equation}
where $\tilde{\Pi}_{\bar{\mathcal{V}}}$ is the projection onto $\mathbf{1}^\perp$ in
$\mathbb{R}^{|\bar{\mathcal{V}}|}$, $\tilde{z}_t$ denotes logits restricted to $\bar{\mathcal{V}}$
up to an additive constant, and $\tilde{\kappa}_{t,\mathrm{traj}}$ is defined from the softmax Jacobian
on the restricted simplex.
\end{proposition}
We emphasize that the existence of such restricted logits is used only to justify the geometry of
the mapping. The analysis does not require access to $\tilde{z}_t$ in practice.

\begin{theorem}[Observable instability lower-bounds internal step change]
\label{thm:observable-lower-bound}
Under (A2), for consecutive steps,
\begin{equation}
  \mathrm{JSD}(p_t,p_{t-1})
  \ge \frac{\kappa_{t,\mathrm{traj}}^2\sigma_W^2}{8}\, \lVert h_t - h_{t-1} \rVert_2^2.
\end{equation}
\end{theorem}
This result formalizes why large consecutive-step distributional shifts are consistent with
substantial internal state changes, supporting our use of $D_t=\mathrm{JSD}(\tilde{p}_t,\tilde{p}_{t-1})$
as an observable proxy for regime shifts.

\paragraph{Entropy as a proxy for decision fragility.}
\begin{definition}[Decision margin]
Let $\Delta_t := z_{t,(1)} - z_{t,(2)}$ be the gap between the largest and second-largest logits
at step $t$.
\end{definition}

\begin{lemma}[Margin robustness]
If an additive logit perturbation $\eta$ satisfies $\lVert \eta \rVert_{\infty} < \tfrac{1}{2}\Delta_t$,
then the argmax token under greedy decoding does not change.
\end{lemma}

\begin{lemma}[Entropy upper-bounds peak confidence]
For any distribution $P$, $H(P) \ge -\log P_{\max}$, where $P_{\max} := \max_i P(i)$.
\end{lemma}
High entropy therefore rules out a highly peaked distribution and is consistent with locally
fragile decisions where several candidates are competitive. This motivates combining distributional
change with an uncertainty term in $\It = D_t + \lambda H_t$. Empirically, this complements $D_t$
under top-$k$ logging, consistent with our $\lambda$ ablation.

\paragraph{Corrective vs.\ destructive instability and timing effects.}
We interpret instability events as transitions between more stable basins of continuation. Let
$\mathcal{A}_G$ denote an empirical basin of internal states that tend to converge to correct answers
and let $\mathcal{A}_B$ denote an empirical basin of internal states that tend to converge to
incorrect answers.
A transition $\mathcal{A}_B \to \mathcal{A}_G$ corresponds to corrective instability, while a transition
$\mathcal{A}_G \to \mathcal{A}_B$ corresponds to destructive instability.

\paragraph{Stabilization assumption.}
We formalize the recoverability intuition with a stylized stabilization-time assumption: entering
(or re-entering) the good basin may require a minimum number of remaining steps before the
trajectory reliably terminates correctly. Let $\mathrm{Correct}$ denote the event that the final
answer at step $T$ is correct. We assume there exist $\taumix\in\mathbb{N}$ and $\delta\in(0,1)$ such
that:
\begin{itemize}
  \item \textbf{(S1) Sufficient budget implies high success.} For any step $t$, if
  $h_t\in\mathcal{A}_G$ and $T-t\ge \taumix$, then
  \begin{equation}
    \Pr(\mathrm{Correct}\mid h_t\in\mathcal{A}_G,\ T-t\ge\taumix)\ \ge\ 1-\delta.
  \end{equation}
  \item \textbf{(S2) Fresh entry with insufficient budget is unreliable.}
  There exists an entry subset $\mathcal{A}_G^{\mathrm{entry}}\subseteq \mathcal{A}_G$ such that if
  $h_t\in\mathcal{A}_G^{\mathrm{entry}}$ and $T-t<\taumix$, then
  \begin{equation}
    \Pr(\mathrm{Correct}\mid h_t\in\mathcal{A}_G^{\mathrm{entry}},\ T-t<\taumix)\ \le\ \delta.
  \end{equation}
\end{itemize}
Intuitively, $\taumix$ captures a finite-horizon stabilization or mixing time needed after a
corrective transition before the trajectory reliably terminates correctly.

\begin{theorem}[Late correction penalty]
\label{thm:late-penalty}
Assume (S1) and (S2). If a corrective transition occurs at step $t^\star$ and
$h_{t^\star}\in \mathcal{A}_G^{\mathrm{entry}}$ with $T-t^\star < \taumix$, then the probability of
producing a correct final answer is at most $\delta$. Conversely, if $h_{t^\star}\in\mathcal{A}_G$
and $T-t^\star\ge \taumix$, then the success probability is at least $1-\delta$.
\end{theorem}
This formalizes the recoverability intuition underlying our empirical timing effect: late instability
peaks are less recoverable given a limited remaining decoding budget.

\paragraph{Summary.}
\begin{table*}[t]
\caption{How the theoretical statements relate to the main empirical findings.}%
\centering
\begin{adjustbox}{max width=\textwidth}
\begin{tabular}{l l}
\toprule
Empirical observation & Theoretical intuition \\
\midrule
JSD spikes predict failure trends & \Cref{thm:observable-lower-bound} \\
High entropy indicates fragile decisions & Lemmas above on margin and entropy \\
Not all instability is harmful & Basin interpretation and timing \\
Late peaks are associated with lower accuracy & \Cref{thm:late-penalty} \\
\bottomrule
\end{tabular}
\end{adjustbox}
\label{tab:theory_summary}
\end{table*}

\begin{figure}[t]
\centering
\includegraphics[width=\linewidth]{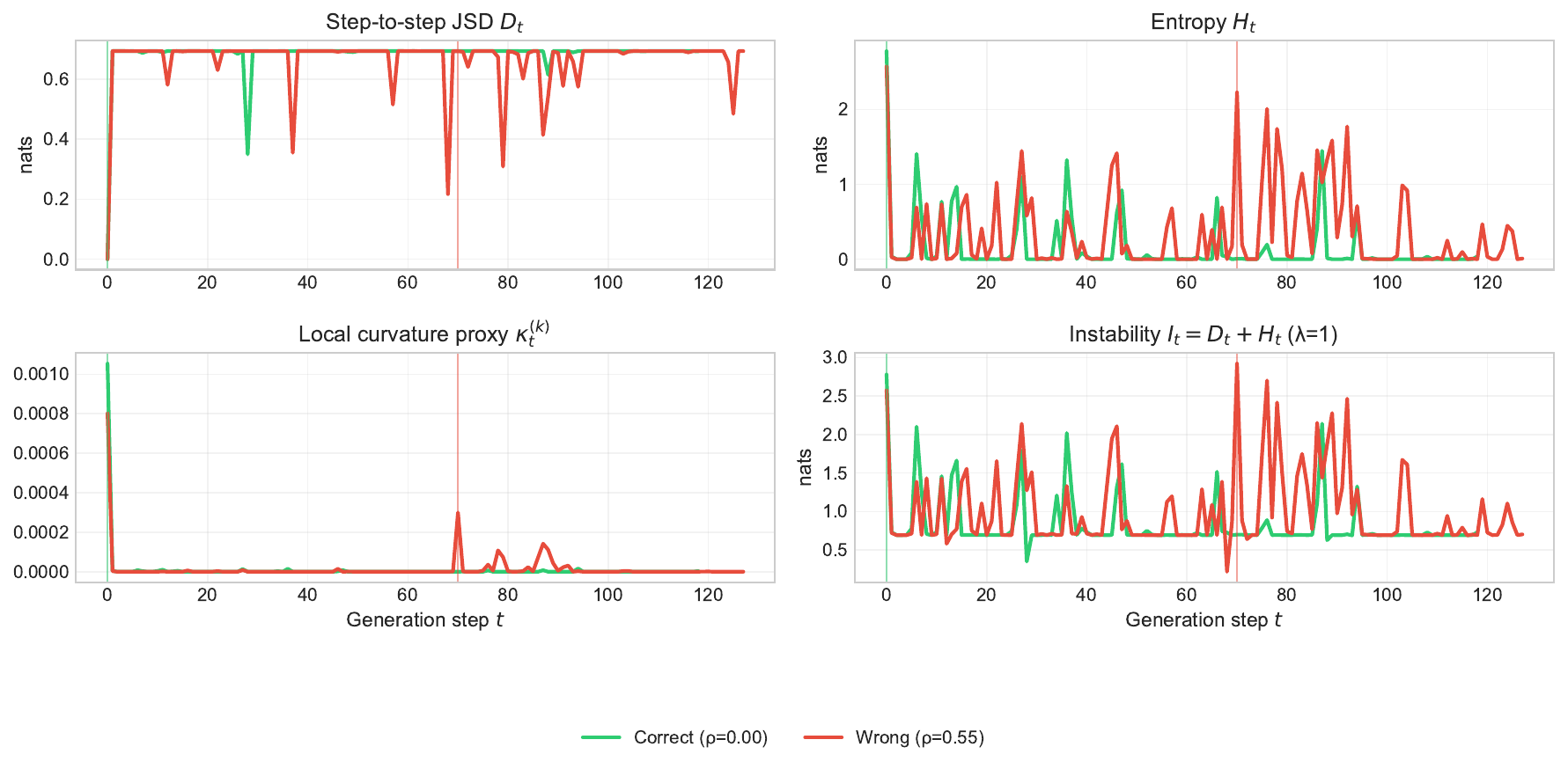}
\caption{Theory and experiment alignment on representative \LlamaFull{} GSM8K greedy traces.
We plot step-to-step JSD $D_t$, entropy $H_t$, the local softmax Jacobian curvature proxy
$\kappa_t^{(k)}$ computed as the smallest non-zero eigenvalue of
$\mathrm{diag}(\tilde{p}_t)-\tilde{p}_t\tilde{p}_t^\top$ (equivalently, the minimum eigenvalue on
$\mathbf{1}^\perp$) on the renormalized top-$k$ distribution,
and the combined instability
$\It=D_t+H_t$ (with $\lambda=1$). The example traces are selected deterministically to illustrate
an early-peak correct trace and a late-peak wrong trace. Large JSD spikes occur during uncertain
segments where $\kappa_t^{(k)}$ is non-negligible, consistent with Lemma~\ref{lem:logit-jsd}, and
late peak position aligns with the recoverability interpretation in Theorem~\ref{thm:late-penalty}.}%
\label{fig:theory_alignment}
\end{figure}

\subsection{\texorpdfstring{Implementation Note: Top-$k$ Truncation}{Implementation Note: Top-k Truncation}}
We use top-$k$ logging because it is available in many black-box inference settings and keeps
storage and compute tractable at per-step resolution. Truncation and renormalization can affect
both entropy and JSD: when consecutive steps have low support overlap (high support-set turnover),
computing divergences on the union of two renormalized top-$k$ supports can amplify spikes relative
to full vocabulary distributions.
Intuitively, when the top candidates change substantially from one step to the next, the truncated
distributions share little support, producing a large JSD spike.

To ensure our conclusions are not artifacts of a particular truncation choice, we run two validation
checks. First, on held-out traces we validate the ``corrective vs.\ destructive'' timing effect
using entropy/JSD computed from full vocabulary logits (no truncation; \cref{sec:corrective-destructive}).
Second, we recompute the main signal from the same logged top-$50$ traces using smaller effective
values of $k$ (by truncating the logged list) and find qualitatively similar separability trends
once $k$ is moderately large (\cref{sec:appendix-topk}).
}

\section{Experiments}
\subsection{Task and Models}
Datasets: GSM8K (test split)~\cite{cobbe2021training} and HotpotQA (distractor validation split)~\cite{yang2018hotpotqa}.
For controlled GSM8K-300 experiments we use the first 300 GSM8K test examples (starting from index
0 for determinism) with \QwenTwoFiveFull~\cite{yang2024qwen2} and \LlamaFull~\cite{grattafiori2024llama3}.
We additionally report dense per-step baseline runs on the full GSM8K test set (1319 examples) and
HotpotQA distractor validation split (7405 examples), including \LlamaThreeBFull{} and \LlamaEightBFull{}~\cite{grattafiori2024llama3}
(\cref{sec:appendix-fullset}). Models are instruction-tuned to reduce formatting confounds; sizes span $\sim$0.5B--8B to test scale
robustness under a fixed diagnostic.

\subsection{Prompting and Decoding}
For GSM8K, prompting and scoring follow the standard format: we prompt with the question text, and
a prediction is scored correct if the final extracted numeric answer matches the reference answer.
We extract the last number in the model output as the predicted answer.
For HotpotQA, we prompt with the question and concatenated context passages and ask for a short,
direct answer; we score correctness by comparing the first line of the model output to the
reference using normalized matching with a containment check.
We deliberately use minimal, zero-shot prompting (no few-shot exemplars) to keep decoding dynamics
comparable and avoid confounds from prompt engineering. As a result, absolute accuracy can be low
for these small models; our focus is on ranking and separability (AUC, monotonic bucket trends),
which is less sensitive to the overall error rate than raw accuracy and is therefore suitable for
diagnostic separability.

We report greedy decoding and stochastic decoding. For GSM8K-300, we evaluate \QwenTwoFiveFull{} with greedy decoding
and evaluate \LlamaFull{} under a small decoding grid with top-$p$ 0.9 and a fixed random seed. For \LlamaFull{},
we additionally evaluate $\temp=0.3$ (top-$p$ 0.9), yielding a small decoding grid
$\temp\in\{0.0,0.3,0.7\}$. In analyses that pool multiple \LlamaFull{} decoding settings (e.g.,
peak-step characteristics), we aggregate across this grid, yielding $3\times 300$ traces in total.

\subsection{Logging and Trace Length Controls}
In our main experiments, we log top-$k$ token log probabilities per-step and cap
generation to 128 new tokens to reduce length confounding. In addition to top-$k$ traces, we
validate key timing findings with a held-out run that logs entropy/JSD from full vocabulary logits (no
truncation; \cref{sec:corrective-destructive}). We also compute early-window versions of our
strength statistic to control for trace length (see Metrics and \cref{sec:appendix-early-window}).
Unless otherwise noted, we log top-$k$ distributions with $k=50$ and report top-$k$ sensitivity in
Appendix~\ref{sec:appendix-topk}.
For timing analyses, we additionally use fixed-window peak position (e.g., $\rho_{50}$) to decouple
peak-timing effects from the total trace length (\cref{sec:appendix-fixed-window}).

\subsection{Metrics}
We evaluate instability as a diagnostic signal rather than an intervention. Let $y \in \{0,1\}$
be the error label ($y=1$ if the final answer is wrong). We report:
\begin{itemize}
  \item \textbf{Bucketed accuracy}: partition examples into five equal-sized quantile buckets by
  instability strength $S$ and report accuracy per bucket.
  \item \textbf{Rank association}: Spearman correlation between $S$ and correctness (encoded as
  $1$ for correct and $0$ for wrong).
  \item \textbf{Separability}: ROC--AUC for predicting $y$ from $S$ (denoted AUC$_{\text{wrong}}$).
\end{itemize}

We additionally compute early-window versions by replacing $S$ with $S_w$ to verify that the
signal is not purely driven by longer generations (a key confound for any max-over-time statistic).

\subsection{Ablations and Controls}
We include a minimal ablation on the instability signal by comparing $\lambda=0$ (JSD only) and
$\lambda=1$ (JSD+entropy). We also report a small decoding grid (temperature $\temp \in \{0.0,0.3,0.7\}$ with
top-$p$ 0.9) to verify robustness under stochasticity.

\section{Results}
\subsection{Instability Strength Correlates Monotonically with Failure Rate}
Instability strength shows a clear monotonic relationship with failure rate.
Across models and decoding settings, bucketed accuracy decreases as instability strength
increases (\cref{fig:bucket_trends}). The monotonic decrease is observed under equal-sized buckets,
rather than an artifact of sample imbalance.
This is consistent with the primary claim that dynamically unstable trajectories are disproportionately
associated with wrong final answers.

\subsection{Instability Strength Separates Correct and Incorrect Reasoning}
Instability strength separates correct from incorrect outputs: Spearman correlations are
consistently negative and AUC$_{\text{wrong}}$ values are consistently above chance
(\cref{tab:robustness}). Early-window signals preserve predictive power,
indicating the effect is not simply driven by longer generations.
Bootstrap confidence intervals over examples support the stability of these trends
(\cref{sec:appendix-bootstrap}).

We reuse the same traces and vary only the analysis window to control for length. Even when
restricted to the first 50 steps, instability remains predictive (e.g., AUC $\approx 0.66$ on
\LlamaFull), indicating that the signal is not an artifact of longer sequences.
We report explicit early-window AUC values (using $S_{50}$) in \cref{sec:appendix-early-window}.
This fixed-window evaluation directly targets the trace-length confound: longer generations provide
more opportunities for a large $\max_t \It$ even under identical per-step dynamics.
For example, using $S_{50}$ yields AUC$_{\text{wrong}}$ of 0.605 on \QwenTwoFiveFull{} greedy and 0.665 on
\LlamaFull{} ($\temp=0.0$) (\cref{tab:early_window_auc}).
This suggests a secondary claim: instability is detectable early, before full reasoning
completion, which is valuable for diagnosis and triage.
We include a full early-warning curve (varying the prefix window length) in
Appendix~\ref{sec:appendix-early-window}.

\newcommand{\EarlyWarningFigureAppendix}{%
\begin{figure}[t]
\centering
\includegraphics[width=0.9\linewidth]{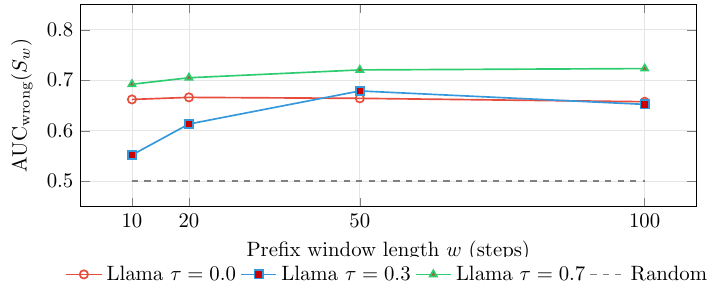}
\caption{Early-warning separability as a function of prefix window length on GSM8K-300.
We report AUC$_{\text{wrong}}(S_w)$ as a function of the window size $w$.
Separability emerges with short prefixes and improves with longer windows, supporting that
instability is detectable before completion of the full reasoning trace.}%
\label{fig:early_warning_auc_curve}
\end{figure}
}

\Cref{tab:robustness} summarizes cross-model and decoding robustness, reporting overall
accuracy alongside the separability of instability strength.

\begin{table*}[t]
\caption{Cross-model and decoding robustness on GSM8K (first 300 test examples; top-$p$ 0.9 for stochastic decoding).
AUC values remain stable across models and decoding settings, supporting a model-agnostic signal.}%
\centering
\begin{adjustbox}{max width=\textwidth}
\begin{tabular}{l l c c c}
\toprule
Model & Decoding & Accuracy & AUC$_{\text{wrong}}$ & Spearman \\
\midrule
\QwenTwoFiveInst{} & greedy & 0.430 & 0.681 & $-0.311$ \\
\LlamaOneBInst{} & $\temp=0.0$ & 0.123 & 0.657 & $-0.178$ \\
\LlamaOneBInst{} & $\temp=0.3$ & 0.123 & 0.667 & $-0.190$ \\
\LlamaOneBInst{} & $\temp=0.7$ & 0.094 & 0.741 & $-0.244$ \\
\bottomrule
\end{tabular}
\end{adjustbox}
\label{tab:robustness}
\end{table*}
Across models and decoding settings (\cref{tab:robustness}), AUC values remain stable. On \LlamaFull,
the monotonic trend appears with AUC values around 0.66 (greedy) to 0.74 (stochastic). The small
decoding grid ($\temp \in \{0.0,0.3,0.7\}$) preserves the trend: increased stochasticity reduces
accuracy but does not remove the instability-failure relationship.

We additionally validate on full GSM8K and HotpotQA runs, and include a ReClor~\cite{yu2020reclorreadingcomprehensiondataset} check as a
representative setting where the signal can weaken.
ReClor is a multiple-choice logical reasoning benchmark; under our short-answer decoding
and first-line matching setup, correct predictions are extremely sparse.
As a result, errors are dominated by stable-but-wrong trajectories, which can induce weak
or even reversed correlations between instability and correctness in this setting.
Detailed tables and figures are reported in Appendix~\ref{sec:appendix-fullset}.

\newcommand{\FullSetRunsAppendix}{%
\begin{table*}[t]
\caption{Dense per-step runs on GSM8K and HotpotQA (full test/validation splits), plus a ReClor
validation subset (300 examples) for \LlamaFull{} under stochastic decoding.
We report separability via AUC$_{\text{wrong}}$ for the full-trace peak strength $S=\max_t \It$ and
an early-window control $S_{50}=\max_{t\le 50}\It$.}%
\centering
\scriptsize
\setlength{\tabcolsep}{4pt}
\begingroup
\rowcolors{2}{gray!10}{white}
\begin{adjustbox}{max width=\textwidth}
\begin{tabular}{@{}l l r c c c c@{}}
\toprule
Model & Dataset & $N$ & Accuracy & \shortstack{AUC$_{\text{wrong}}$\\$(S)$} & Spearman & \shortstack{AUC$_{\text{wrong}}$\\$(S_{50})$} \\
\midrule
\LlamaOneBInst{} & GSM8K & 1319 & 0.434 & 0.673 & $-0.298$ & 0.615 \\
\LlamaThreeBInst{} & GSM8K & 1319 & 0.707 & 0.687 & $-0.295$ & 0.618 \\
\LlamaEightBInst{} & GSM8K & 1319 & 0.749 & 0.692 & $-0.288$ & 0.598 \\
\QwenHalfBInst{} & GSM8K & 1319 & 0.301 & 0.714 & $-0.341$ & 0.654 \\
\QwenTwoFiveInst{} & GSM8K & 1319 & 0.416 & 0.643 & $-0.244$ & 0.565 \\
\QwenTwoFiveSevenBInst{} & GSM8K & 1319 & 0.500 & 0.644 & $-0.250$ & 0.581 \\
\LlamaOneBInst{} & HotpotQA & 7405 & 0.483 & 0.594 & $-0.163$ & 0.594 \\
\LlamaThreeBInst{} & HotpotQA & 7405 & 0.632 & 0.662 & $-0.271$ & 0.662 \\
\LlamaEightBInst{} & HotpotQA & 7405 & 0.704 & 0.675 & $-0.277$ & 0.675 \\
\QwenHalfBInst{} & HotpotQA & 7405 & 0.410 & 0.568 & $-0.116$ & 0.575 \\
\QwenTwoFiveInst{} & HotpotQA & 7405 & 0.537 & 0.603 & $-0.179$ & 0.603 \\
\QwenTwoFiveSevenBInst{} & HotpotQA & 7405 & 0.675 & 0.654 & $-0.250$ & 0.654 \\
\LlamaOneBInst{} & ReClor & 300 & 0.027 & 0.054 & $0.250$ & 0.046 \\
\bottomrule
\end{tabular}
\end{adjustbox}
\endgroup
\label{tab:fullset_runs}
\end{table*}

\begin{figure}[t]
\centering
\includegraphics[width=\linewidth]{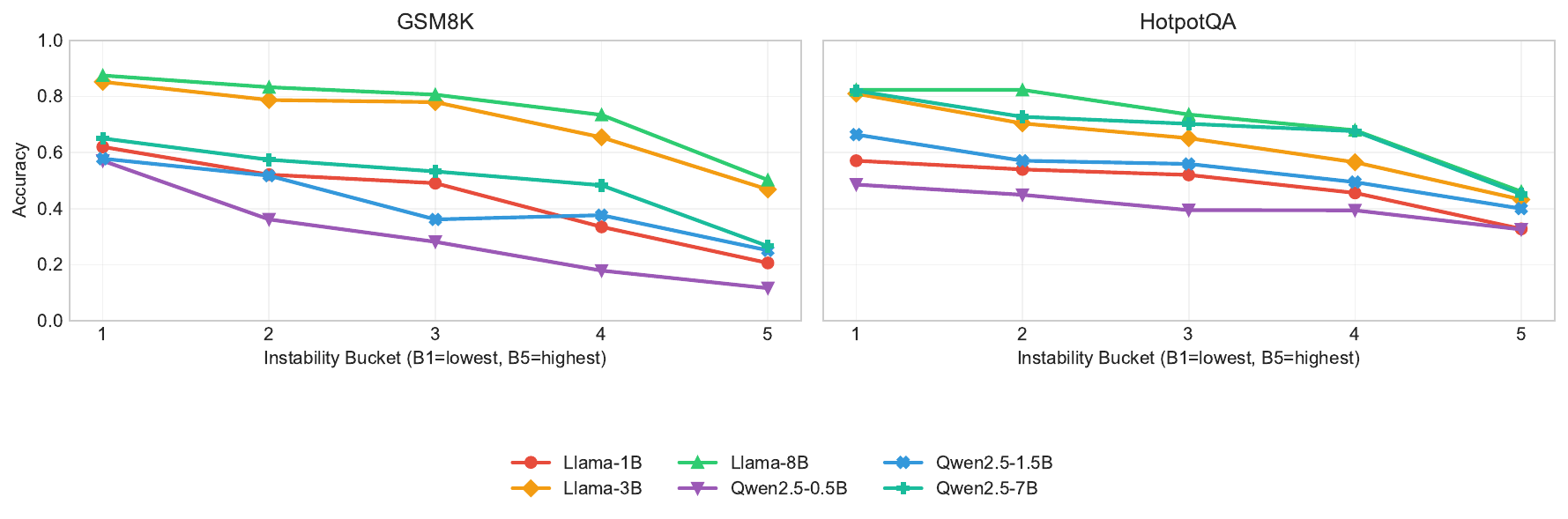}
\caption{Bucketed accuracy trends for the full-set runs in \cref{tab:fullset_runs}.
Each curve plots accuracy across five equal-sized instability buckets (B1 to B5), showing a
monotonic decline as instability increases on both tasks and for all models.}%
\label{fig:fullset_bucket_trends}
\end{figure}

\begin{figure}[t]
\centering
\includegraphics[width=\linewidth]{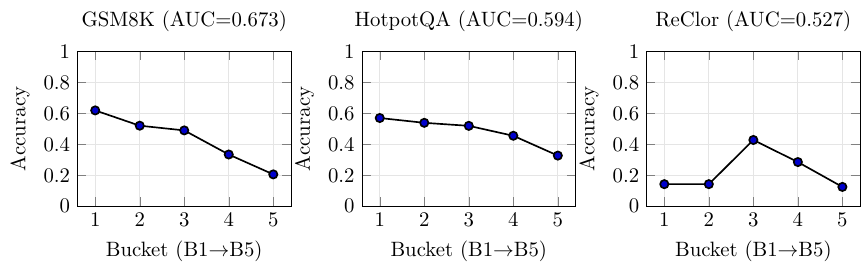}
\caption{Cross-task validation on a fixed model (Llama-3.2-1B).
We plot bucketed accuracy trends (B1 to B5) by instability strength on GSM8K, HotpotQA, and ReClor.
On GSM8K and HotpotQA, accuracy declines with instability. On ReClor in this short-answer
multiple-choice setting, the trend inverts: most errors are stable-but-wrong, and the rare correct
traces exhibit higher instability.}%
\label{fig:cross_task_bucket_trends_llama1b}
\end{figure}
}

\subsection{Not All Instability Is Harmful}
Instability magnitude alone is insufficient: some high-instability episodes reflect beneficial
self-correction rather than failure. We refine diagnosis by distinguishing \emph{corrective} versus
\emph{destructive} instability via the timing of the peak (Section~\ref{sec:corrective-destructive}).

\subsection{Signal Ablation}
Using JSD alone ($\lambda=0$) collapses the strength distribution and yields near-random
predictive power (AUC $\approx 0.52$) on \LlamaFull, while the combined signal
($\lambda=1$) restores clear separation and monotonicity. This supports the need to combine
distributional change with uncertainty.

We observe that a JSD-only signal ($\lambda=0$) becomes nearly constant on \LlamaFull,
yielding near-random separability. A plausible explanation is saturation under top-$k$
approximations: when successive steps have weak support overlap, renormalized top-$k$
distributions can reduce the dynamic range of divergence estimates, reducing sensitivity to more
graded changes in instability. Adding entropy restores a robust signal by capturing uncertainty and competition
among alternatives even when support overlap is limited. This ablation indicates that
uncertainty is a necessary component of our instability proxy.
We also explored change-point detectors, but found continuous strength more stable as a diagnostic
signal; see \cref{sec:appendix-detector}.
We additionally report time-structure negative controls and entropy-family baselines on the same
traces; see \cref{sec:appendix-controls}.
We include additional baseline tables and ablation figures in Appendix~\ref{sec:appendix-controls}.

\newcommand{\EntropyFamilyBaselinesTable}{%
\begin{table*}[t]
\caption{Entropy-family baselines on GSM8K (300 examples) for greedy decoding.
JSD-only statistics are near chance, consistent with our $\lambda$ ablation.}%
\centering
\small
\setlength{\tabcolsep}{4pt}
\begin{adjustbox}{max width=\textwidth}
\begin{tabular}{l l c c}
\toprule
Model & Statistic & AUC$_{\text{wrong}}(S)$ & Spearman$(S)$ \\
\midrule
\LlamaFull{} & $S_I=\max_t \It$ & 0.657 & $-0.178$ \\
\LlamaFull{} & $S_H=\max_t H_t$ & 0.650 & $-0.171$ \\
\LlamaFull{} & $S_{\Delta H}=\max_t |H_t-H_{t-1}|$ & 0.563 & $-0.072$ \\
\LlamaFull{} & $S_D=\max_t D_t$ & 0.511 & $-0.015$ \\
\bottomrule
\end{tabular}
\end{adjustbox}
\label{tab:entropy_family_baselines_s}
\end{table*}
}

\begin{figure}[t]
\centering
\includegraphics[width=0.9\linewidth]{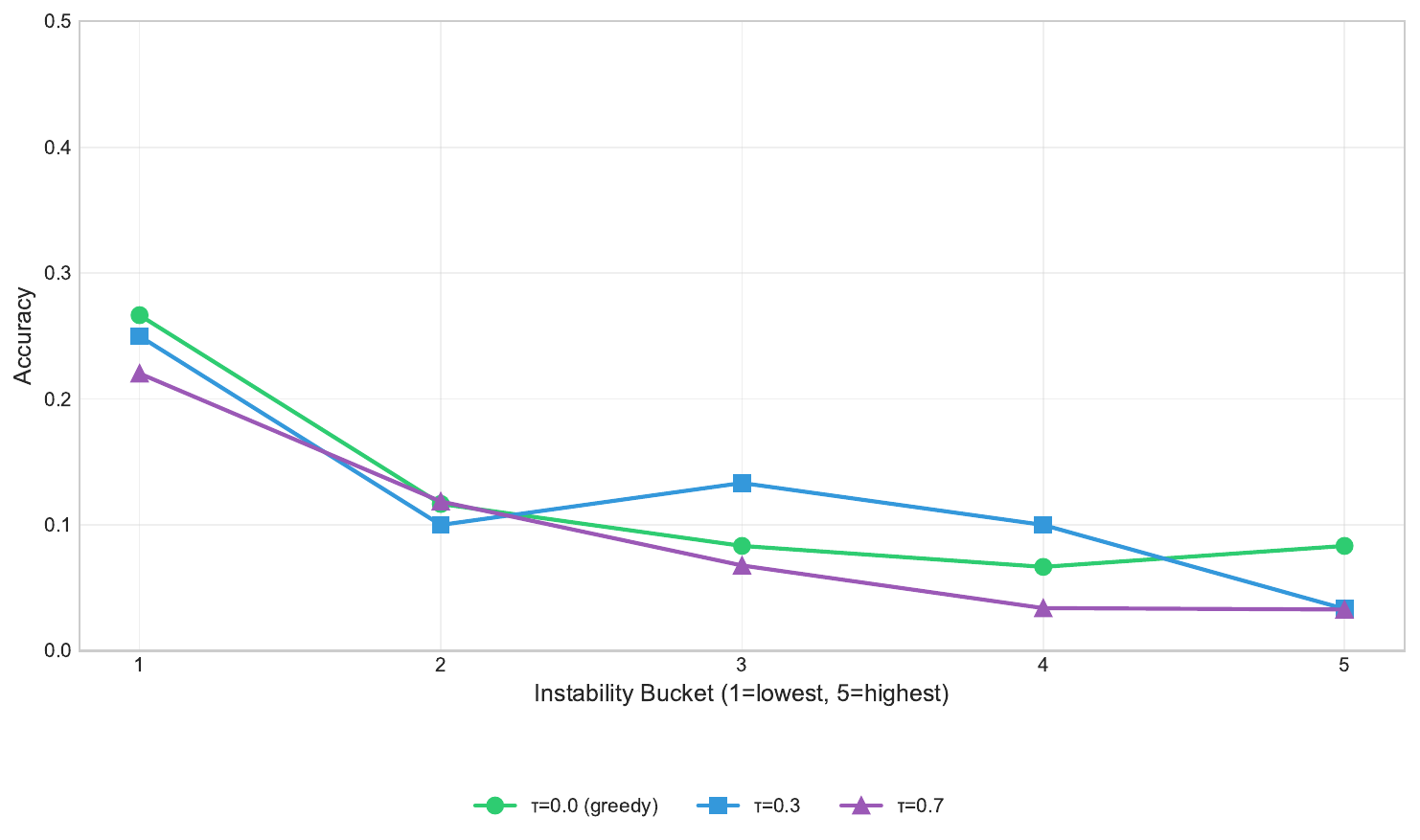}
\caption{Monotonic bucket trends for \LlamaFull{} across temperature settings
($\temp \in \{0.0, 0.3, 0.7\}$). Higher instability buckets consistently show lower accuracy,
confirming the predictive value of the instability signal.}%
\label{fig:bucket_trends}
\end{figure}

\newcommand{\LambdaAblationFigure}{%
\begin{figure}[t]
\centering
\includegraphics[width=0.9\linewidth]{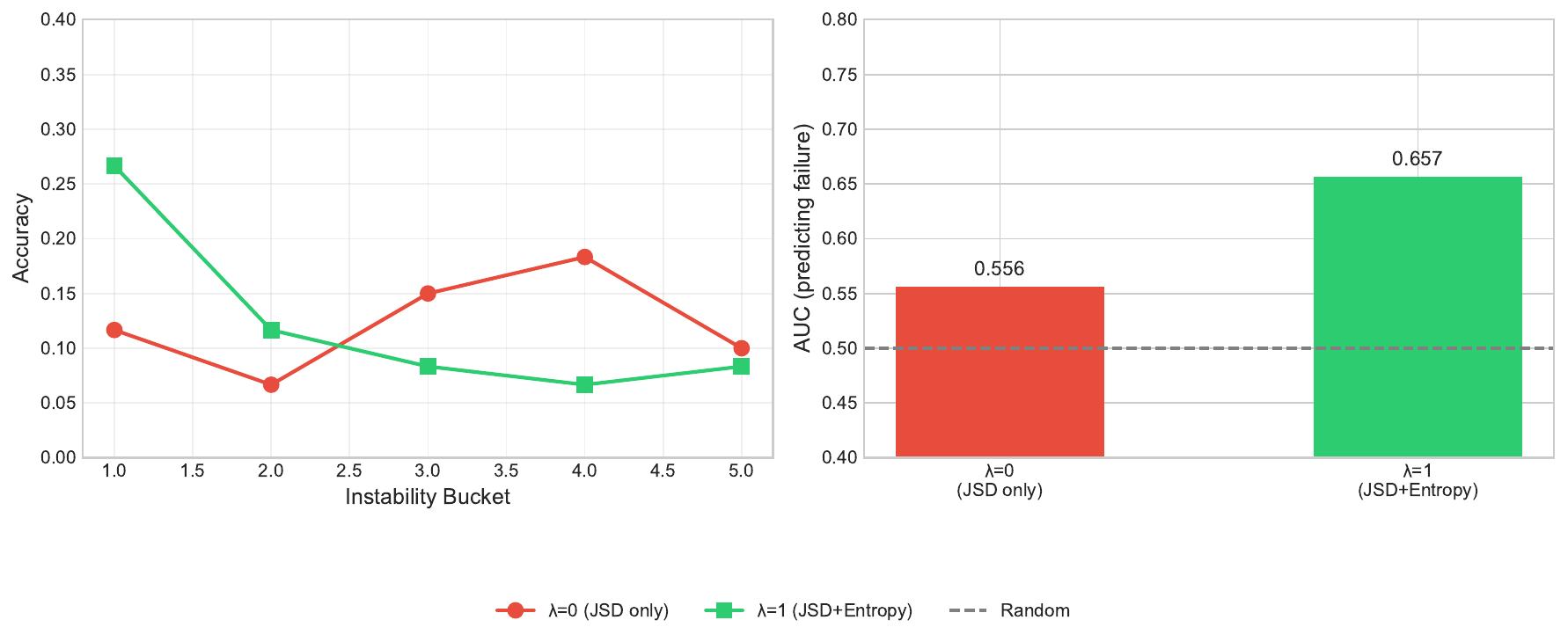}
\caption{Signal ablation on \LlamaFull: comparing JSD-only ($\lambda=0$) vs.\ JSD+Entropy
($\lambda=1$). The entropy component improves discrimination, particularly in lower-instability
buckets where pure JSD may saturate.}%
\label{fig:lambda_ablation}
\end{figure}
}

\newcommand{\FailureModeFigureAppendix}{%
\begin{figure}[t]
\centering
\includegraphics[width=0.85\linewidth]{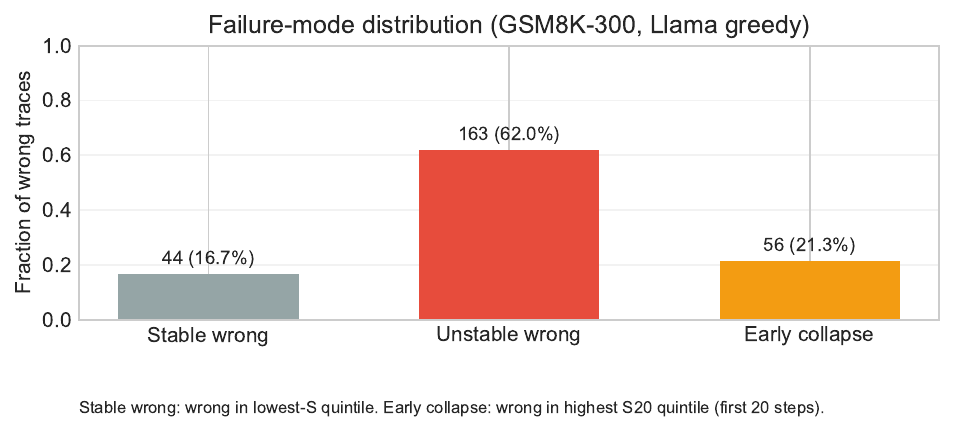}
\caption{Failure-mode distribution on GSM8K (300 examples) for \LlamaFull{} greedy.
We partition wrong traces into three disjoint categories for interpretability: \emph{stable wrong}
(wrong traces in the lowest quintile of instability strength $S$), \emph{early collapse}
(wrong traces in the highest quintile of early-window strength $S_{20}=\max_{t\le 20} \It$), and
\emph{unstable wrong} (the remaining wrong traces).
This breakdown highlights that a non-trivial fraction of errors are stable wrong, consistent with
the limitation that instability is not a universal explanation of failure.}%
\label{fig:failure_mode_distribution}
\end{figure}
}

\newcommand{\PeakPositionFigureAppendix}{%
\begin{figure}[t]
\centering
\includegraphics[width=0.75\linewidth]{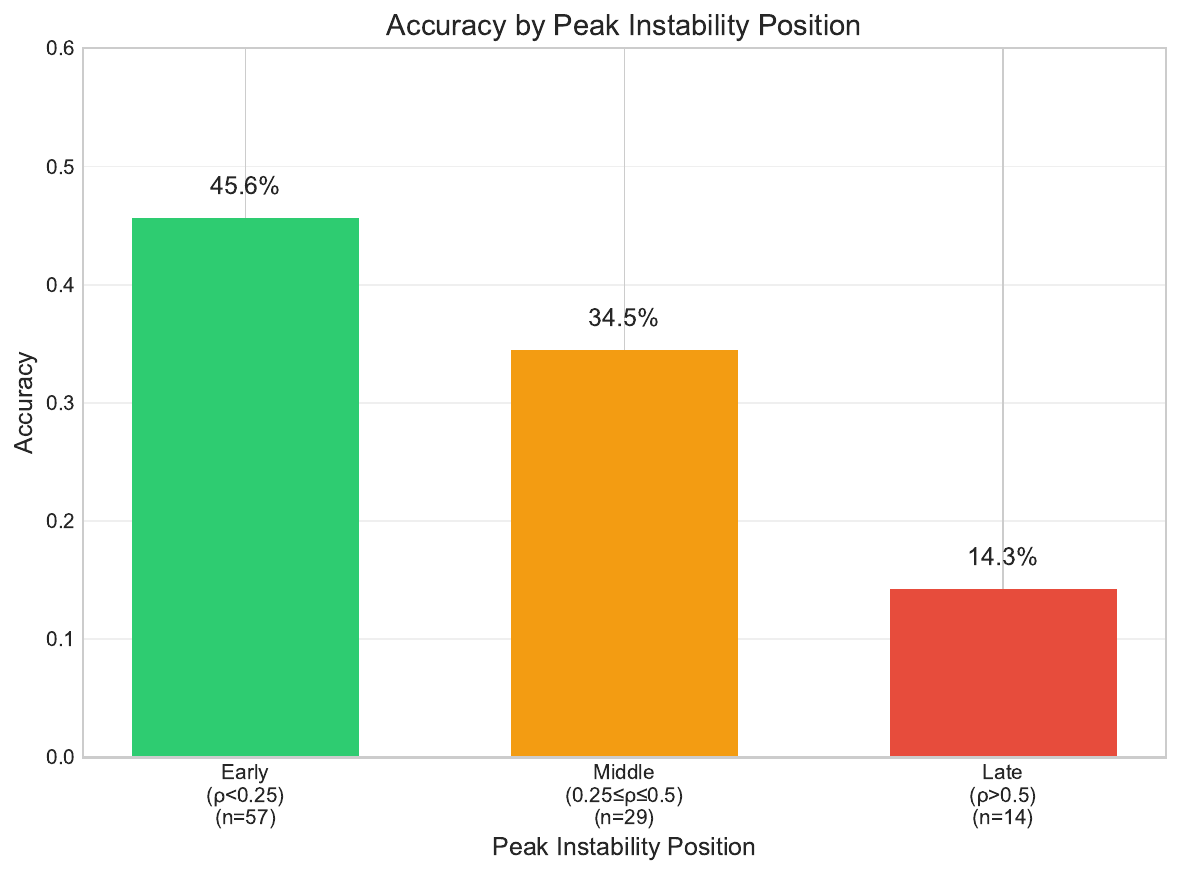}
\caption{Accuracy by peak instability position on the 100-trace held-out baseline run (same setting as
\Cref{tab:corrective-destructive}). Early peaks allow recovery and yield higher accuracy; late peaks
leave insufficient time to recover.}%
\label{fig:corrective-destructive}
\end{figure}
}

\section{Corrective vs.\ Destructive Instability}
\label{sec:corrective-destructive}

A critical observation emerges from our analysis: not all high-instability episodes are
harmful. We identify two qualitatively distinct regimes of dynamic instability with opposite
implications for reasoning outcomes.
These regimes are defined jointly by instability events and recoverability, rather than by
instability magnitude alone.

\smallskip
\subsection{Two Regimes of Instability}

\paragraph{Destructive instability.}
Destructive instability refers to high-instability episodes that are followed by continued
instability and incorrect final answers. In these traces, uncertainty tends to rise and the
trajectory does not stabilize onto a coherent solution path.

\paragraph{Corrective instability.}
Corrective instability refers to high-instability episodes that are followed by stabilization and
correct final answers. In these traces, uncertainty often decreases after the peak and the
trajectory converges.

\subsection{Operationalization}
We operationalize recoverability using the \textbf{peak instability position}, which captures
when the strongest instability occurs within a reasoning trace.

\smallskip
Given a trace $\{\tilde{p}_t\}_{t=1}^T$, let $t^\star = \arg\max_t \It$ (breaking ties by the smallest $t$) denote the step of peak
instability. We define the relative peak position as $\rho = t^\star / T$, where $T$ is the
trace length, and classify traces as:
\begin{itemize}
  \item \textbf{Early peak} ($\rho < 0.25$): Instability occurs early, leaving time for
  subsequent recovery; this regime is characteristic of corrective instability.
  \item \textbf{Late peak} ($\rho > 0.5$): Instability occurs late, with insufficient time
  for recovery; this regime is characteristic of destructive instability.
\end{itemize}
We adopt coarse thresholds for interpretability (clear ``early'' vs.\ ``late'' separation);
the qualitative ordering is robust to alternative threshold pairs and to a binned analysis of $\rho$
(\cref{sec:appendix-rho-thresholds}), and also holds under a fixed-window peak definition
(\cref{sec:appendix-fixed-window}).

This operationalization captures a simple intuition: when difficulty arises early, the model
may still recover through subsequent stable reasoning, whereas late instability leaves limited
opportunity for correction. Importantly, $\rho$ is descriptive rather than causal; we do not
claim that peak position \emph{causes} success or failure. As a control against the trivial
``late means no time left'' critique, we additionally measure peak position within a fixed early
window and observe the same qualitative ordering (\cref{sec:appendix-fixed-window}).

\smallskip
\subsection{Empirical Evidence}

We partition traces by peak instability position and compare their final-answer accuracy on a
held-out baseline run of \LlamaFull{} (100 traces, logged without top-$k$ truncation for this
validation).

\begin{table}[t]
\caption{Accuracy by peak instability position on the held-out baseline run (\LlamaFull, GSM8K, 100 traces).
Early-peak traces achieve more than $3\times$ the accuracy of late-peak traces, despite both
exhibiting high instability at some point.}%
\centering
\begin{adjustbox}{max width=\columnwidth}
\begin{tabular}{l c c}
\toprule
Peak position & \% of traces & Accuracy \\
\midrule
Early ($\rho < 0.25$) & 57\% & 0.46 \\
Middle ($0.25 \le \rho \le 0.5$) & 29\% & 0.35 \\
Late ($\rho > 0.5$) & 14\% & 0.14 \\
\bottomrule
\end{tabular}
\end{adjustbox}
\label{tab:corrective-destructive}
\end{table}

\Cref{tab:corrective-destructive} shows a clear monotonic relationship between peak position and
accuracy. Traces where instability peaks early have substantially higher accuracy (46\%) than those
where it peaks late (14\%), suggesting that instability magnitude alone is an incomplete diagnostic:
the \emph{timing} and \emph{recoverability} of instability matter.
To characterize peak steps beyond timing, we also report two ``low-cost probes'' (margin drops and
support-set turnover) in Appendix~\ref{sec:appendix-peak-probes}.
At the peak step, corrective traces tend to exhibit sharper probability margin drops, while
destructive traces exhibit higher support-set turnover, consistent with decisive revision versus
chaotic route switching (Appendix~\ref{sec:appendix-peak-probes}).

We provide a visualization in Appendix~\ref{sec:appendix-additional-figures} (\Cref{fig:corrective-destructive}).

\subsection{Implications}




This distinction has two important implications for how inference-time instability should be
interpreted. First, raw instability strength ($S = \max_t \It$) remains a useful aggregate predictor
of failure, but peak position provides a complementary summary that separates traces which recover
from those that do not. This highlights that temporal structure carries information beyond magnitude
alone.

Second, our analysis clarifies the scope of this work. We do not propose or evaluate interventions;
instead, our goal is diagnosis and mechanistic characterization based solely on black-box
inference-time observables.

Taken together, these findings refine our central claim: while dynamic instability is predictive of
failure, the underlying mechanism is nuanced. Destructive instability drives most of the predictive
signal, whereas corrective instability acts as a confound that, once accounted for, sharpens the
diagnostic.

\section{Discussion}
\paragraph{Takeaways.}
\textbf{What we showed:} inference-time instability can be detected from black-box signals (top-$k$
token distributions) and reliably predicts reasoning failure, with consistent monotonic bucket
trends and above-chance AUC across models, decoding settings, and tasks (including full-set
validation on \LlamaThreeBFull{} and \LlamaEightBFull{}).
\textbf{What we clarified:} not all instability is harmful; its diagnostic meaning depends on
timing, with early peaks more often followed by recovery and late peaks more often followed by
failure.
\textbf{What we did not claim:} we do not claim a unique internal mechanism, causality, or
controllability; this paper is diagnosis-only and does not propose stabilization operators.
Taken together, our results suggest that inference-time instability provides a reliable diagnostic
signal for reasoning behavior, while its consequences depend critically on when it arises.

\subsection{Failure Taxonomy}
Our framing separates at least three conceptual failure modes.
(1) \emph{Stable-but-wrong} failures produce incorrect answers despite low instability strength $S$,
often consistent with knowledge gaps or spurious heuristics.
(2) \emph{Dynamically unstable} failures exhibit high instability strength and cascading divergence,
which is the focus of this paper.
(3) \emph{Early-collapse} failures exhibit high early-window instability (e.g., large
$S_{20}=\max_{t\le 20}\It$) and are often tied to formatting errors or decoding pathologies. This
differs from early-peak-but-recoverable cases in \cref{sec:corrective-destructive}, where early
instability is followed by stabilization and correct completion. This taxonomy clarifies that
instability is a specific diagnostic dimension, not a universal explanation of failure.

We provide a breakdown figure in Appendix~\ref{sec:appendix-additional-figures} (\Cref{fig:failure_mode_distribution}).

\subsection{Mechanistic Interpretation}
The instability signal is purely inference-time and black-box: it uses only per-step token
probabilities. Mechanistically, spikes can reflect \emph{route switching} (abrupt changes in which
continuation is locally preferred) and \emph{decision fragility} (near ties among candidates).
We provide a longer interpretation, together with the supporting ``low-cost probes'', in
Appendix~\ref{sec:appendix-mechanistic}.

\section{Limitations}
First, not all failures are dynamic; stable-but-wrong cases exist and are not addressed by
instability signals. Second, change-point detection is sensitive to detector choice and
thresholding, making binary change-point labels less reliable than continuous strength.
Third, our approach assumes that relative changes in model probabilities carry meaningful
information about underlying state transitions. While we do not require perfect probability
calibration, future work can study how calibration quality affects this diagnostic signal.
Finally, this study focuses on a single primary task (GSM8K) for controlled grid experiments on two small
models, with additional full-set runs on GSM8K and HotpotQA that include \LlamaThreeBFull{} and \LlamaEightBFull{}.
Broader evaluation across tasks and larger models is left to future work but is not necessary to
validate the central mechanism. These limitations are inherent to a diagnostic perspective and
do not detract from the goal of identifying dynamically unstable reasoning failures.
\paragraph{When the signal weakens.}
The signal targets dynamically unstable failure modes and can be weaker when \emph{stable-but-wrong}
errors dominate. Concretely, on ReClor in our setting, higher instability tends to correlate with
correctness (AUC$_{\text{wrong}}<0.5$). The signal also weakens when the logged top-$k$ list is too
small to estimate entropy/JSD reliably (Appendix~\ref{sec:appendix-topk}).

\section{Conclusion}
We provide a simple and reproducible inference-time diagnostic for dynamic instability in LLM
reasoning. Instability strength yields a robust, monotonic relationship with failure
rate and predicts errors with AUC in the range 0.57 to 0.78 across GSM8K and HotpotQA, using
controlled GSM8K-300 runs and dense baseline runs at full-set scale.

Crucially, we show that not all instability is harmful: \emph{corrective} instability reflects
beneficial self-correction and is associated with higher accuracy, while \emph{destructive}
instability drives the failure-predictive signal.
Peak-step probes further indicate that corrective traces exhibit sharper margin drops, while
destructive traces exhibit higher support-set turnover.
We report these ``low-cost probes'' and a longer mechanistic interpretation in
Appendix~\ref{sec:appendix-mechanistic} and Appendix~\ref{sec:appendix-peak-probes}.
Taken together, these results establish instability strength and peak position as robust, black-box
diagnostics of reasoning breakdowns.

\section*{Impact Statement}

This work aims to advance the understanding of inference-time behavior in large language models by introducing a diagnostic signal for identifying dynamic instability during reasoning. 
The proposed method is purely observational, operates entirely at inference time, requires no model modification or retraining, and is intended as an analysis tool rather than a corrective or control mechanism.

Potential positive impacts include improved transparency and diagnostic assessment of LLM reasoning processes, which may support evaluation, debugging, and risk-aware deployment, particularly in settings where reasoning failures are costly. 
By distinguishing between corrective and destructive instability, the work also cautions against interpreting all forms of instability as uniformly harmful.

We emphasize that the method is not designed for influencing model behavior, automating decisions, or surveilling users, and it does not provide mechanisms for steering or modifying model outputs. 
Accordingly, we do not identify new ethical risks beyond those already associated with large language models. Any downstream use of inference-time diagnostics should be considered within appropriate task-specific and human-in-the-loop evaluation frameworks.

\bibliographystyle{icml2026}
\bibliography{refs}

\clearpage
\appendix
\section{Appendix: Additional Figures}
\label{sec:appendix-additional-figures}
\paragraph{Peak-timing visualization.}
\Cref{fig:corrective-destructive} visualizes the peak-timing effect reported in
\cref{sec:corrective-destructive}: earlier peak positions are more recoverable and are associated
with higher final-answer accuracy, while later peaks are associated with substantially lower
accuracy, consistent with the ``corrective vs.\ destructive'' interpretation.
\PeakPositionFigureAppendix
\paragraph{Failure-mode breakdown.}
\Cref{fig:failure_mode_distribution} provides a complementary view of error types in the same
GSM8K-300 setting. It summarizes the stable-but-wrong, early-collapse, and dynamically unstable
categories used in \cref{sec:appendix-additional-figures} and discussed in \cref{sec:appendix-mechanistic}.
\FailureModeFigureAppendix

\section{Appendix: Theoretical Details}
\label{sec:appendix-theory}
\FullTheoryDetails

\section{Appendix: Proofs}
\label{sec:appendix-proofs}
This appendix provides proofs for the theoretical statements in Appendix~\ref{sec:appendix-theory}.

\subsection{Preliminaries: Softmax shift invariance}
Softmax is invariant to adding constants: $\mathrm{softmax}(z)=\mathrm{softmax}(z+c\mathbf{1})$ for
any scalar $c$. We therefore measure logit differences on the identifiable subspace
$\mathbf{1}^\perp := \{v \in \mathbb{R}^V : \langle v,\mathbf{1}\rangle = 0\}$ using the projection
\begin{equation}
  \Pi := I - \frac{1}{V}\mathbf{1}\mathbf{1}^\top, \qquad \lVert u \rVert_{\perp} := \lVert \Pi u \rVert_2.
\end{equation}

\subsection{Proof of projected logit lower bound on Jensen--Shannon divergence}
\begin{proof}[Proof of Lemma~\ref{lem:logit-jsd}]
Fix consecutive logits $z$ and $z'$ with $p=\mathrm{softmax}(z)$ and $q=\mathrm{softmax}(z')$.
Define the line segment $z(s)=z'+s(z-z')$ for $s\in[0,1]$ and let $p(s)=\mathrm{softmax}(z(s))$.
Define
\begin{equation}
  \kappa_{t,\mathrm{traj}}
  :=
  \inf_{s\in[0,1]}
  \lambda_{\min}\!\Big(J(p(s))\Big|_{\mathbf{1}^\perp}\Big),
\end{equation}
where $J(p)=\mathrm{diag}(p)-pp^\top$ is the softmax Jacobian. By definition, for any
$v\in\mathbf{1}^\perp$ and any $s\in[0,1]$ we have $v^\top J(p(s))v \ge \kappa_{t,\mathrm{traj}}\lVert v\rVert_2^2$.
Let $m=\tfrac{1}{2}(p+q)$. By definition,
\begin{equation}
  \mathrm{JSD}(p,q) = \tfrac{1}{2}\mathrm{KL}(p\|m)+\tfrac{1}{2}\mathrm{KL}(q\|m).
\end{equation}
Pinsker's inequality (constants differ by convention) gives $\mathrm{KL}(p\|m) \ge 2\,\mathrm{TV}(p,m)^2$ and
$\mathrm{KL}(q\|m) \ge 2\,\mathrm{TV}(q,m)^2$, where $\mathrm{TV}(a,b)=\tfrac{1}{2}\lVert a-b\rVert_1$
(and we use natural logarithms). Any equivalent Pinsker constant convention would lead to the same
qualitative conclusion up to universal factors.
Since $p-m=\tfrac{1}{2}(p-q)$ and $q-m=\tfrac{1}{2}(q-p)$, we have
\begin{equation}
  \mathrm{JSD}(p,q) \ge \tfrac{1}{2}\mathrm{TV}(p,q)^2 = \tfrac{1}{8}\lVert p-q\rVert_1^2
  \ge \tfrac{1}{8}\lVert p-q\rVert_2^2.
\end{equation}

It remains to lower-bound $\lVert p-q\rVert_2$ by $\lVert z-z'\rVert_{\perp}$. Consider the path
$z(s)=z'+s(z-z')$ for $s\in[0,1]$ and define $p(s)=\mathrm{softmax}(z(s))$. Then
\begin{equation}
  \begin{aligned}
    p-q &= p(1)-p(0) \\
    &= \int_0^1 \frac{d}{ds}p(s)\,ds \\
    &= \int_0^1 J(p(s))(z-z')\,ds,
  \end{aligned}
\end{equation}
where $J(p)=\nabla_z\mathrm{softmax}(z)=\mathrm{diag}(p)-pp^\top$ is the softmax Jacobian.
Because $J(p)\mathbf{1}=0$, only the projected component matters:
$J(p(s))(z-z')=J(p(s))\Pi(z-z')$. Let $v:=\Pi(z-z')\in\mathbf{1}^\perp$. Then
\begin{equation}
  p-q = \int_0^1 J(p(s))v\,ds.
\end{equation}
Taking the inner product with $v$ and using symmetry of $J(p(s))$ yields
\begin{equation}
  \begin{aligned}
    \langle p-q, v\rangle
    &= \int_0^1 v^\top J(p(s))v\,ds \\
    &\ge \int_0^1 \kappa_{t,\mathrm{traj}} \lVert v\rVert_2^2\,ds \\
    &= \kappa_{t,\mathrm{traj}} \lVert v\rVert_2^2,
  \end{aligned}
\end{equation}
Since $J(p(s))\mathbf{1}=0$, we interpret $\lambda_{\min}(J(p(s))|_{\mathbf{1}^\perp})$ as the smallest
eigenvalue over unit vectors orthogonal to $\mathbf{1}$.
By Cauchy-Schwarz, $\lVert p-q\rVert_2 \lVert v\rVert_2 \ge \langle p-q, v\rangle$, hence
\begin{equation}
  \lVert p-q\rVert_2 \ge \kappa_{t,\mathrm{traj}} \lVert v\rVert_2 = \kappa_{t,\mathrm{traj}} \lVert z-z'\rVert_{\perp}.
\end{equation}
Combining the two bounds gives
\begin{equation}
  \mathrm{JSD}(p,q) \ge \tfrac{1}{8}\lVert p-q\rVert_2^2 \ge \tfrac{\kappa_{t,\mathrm{traj}}^2}{8}\lVert z-z'\rVert_{\perp}^2,
\end{equation}
as claimed.
\end{proof}

\subsection{Proof of the late correction penalty}
Let $\mathrm{Correct}$ denote the event that the final answer is correct at step $T$.

\begin{proof}[Proof of Theorem~\ref{thm:late-penalty}]
Let $t^\star$ be the step at which a corrective transition occurs, so that
$h_{t^\star}\in \mathcal{A}_G^{\mathrm{entry}}$. Assume $T-t^\star<\taumix$. By (S2), for any
realization of $h_{t^\star}\in\mathcal{A}_G^{\mathrm{entry}}$,
let $E$ denote the event $T-t^\star<\taumix$. Then
\begin{equation}
  \Pr(\mathrm{Correct}\mid h_{t^\star},\ E)\ \le\ \delta.
\end{equation}
Taking expectation over the randomness of $h_{t^\star}$ and applying the law of total probability,
\begin{equation}
  \begin{aligned}
    \Pr(\mathrm{Correct}\mid E)
    &= \mathbb{E}_{h_{t^\star}}\!\left[
      \Pr(\mathrm{Correct}\mid h_{t^\star},\ E)
    \right] \\
    &\le \delta.
  \end{aligned}
\end{equation}
This proves the late correction penalty. Conversely, if $h_{t^\star}\in\mathcal{A}_G$ and
$T-t^\star\ge\taumix$, then (S1) yields
\begin{equation}
  \begin{aligned}
    \Pr(\mathrm{Correct}\mid h_{t^\star}\in\mathcal{A}_G,\ T-t^\star\ge\taumix)
    &\ge 1-\delta.
  \end{aligned}
\end{equation}
\end{proof}

\section{Appendix: Detector Comparison}
\label{sec:appendix-detector}
The derivative-based change-point detector produces frequent detections, while CUSUM is much
more conservative under the current thresholds. In our experiments, the binary change-point
label is less predictive than continuous strength, reinforcing the decision to focus on
strength-based analyses in the main text.

\section{\texorpdfstring{Appendix: Top-$k$ Sensitivity}{Appendix: Top-k Sensitivity}}
\label{sec:appendix-topk}
To validate truncation effects, we recompute the instability strength $S=\max_t \It$ from the
same logged top-$50$ traces using smaller effective values of $k$ by truncating the stored list.
\Cref{tab:topk_sensitivity} shows that the diagnostic signal weakens for very small $k$ but is
qualitatively stable once $k$ is moderately large.

\begin{table*}[t]
\caption{Top-$k$ sensitivity on \LlamaFull{} greedy (GSM8K, 300 traces). Larger $k$ yields a more
stable estimate of entropy/JSD and slightly stronger separability; the overall trend remains.}%
\centering
\begin{adjustbox}{max width=\textwidth}
\begin{tabular}{c c c}
\toprule
Effective $k$ & AUC$_{\text{wrong}}$ & Spearman \\
\midrule
10 & 0.626 & $-0.144$ \\
20 & 0.635 & $-0.154$ \\
50 & 0.657 & $-0.178$ \\
\bottomrule
\end{tabular}
\end{adjustbox}
\label{tab:topk_sensitivity}
\end{table*}

\paragraph{Top-$k$ sensitivity across models.}
As an additional check, we recompute the diagnostic from full-set GSM8K baseline runs while varying the diagnostic truncation level $k$.
\Cref{tab:topk_multimodel} shows that AUC$_{\text{wrong}}$ and rank association change only slightly between $k\in\{20,50,100\}$ across all evaluated models, supporting that our conclusions are not artifacts of a particular top-$k$ choice.
\begin{table*}[t]
\caption{Diagnostic top-$k$ sensitivity across models on GSM8K (full test set, $N{=}1319$).
We recompute $S=\max_t I_t$ using different diagnostic truncation levels $k$ and report separability.
All quantities are computed from the logged top-$k$ distributions at inference time.
Within each model block, shaded cells indicate the best value across diagnostic $k$ for that metric (higher AUC; more negative Spearman).}
\centering
\small
\setlength{\tabcolsep}{3pt}
\begin{adjustbox}{max width=\textwidth}
\begin{tabular}{@{}l c c c c@{}}
\toprule
Model & diag $k$ & Acc & AUC$_{\text{wrong}}(S)$ & Spearman$(S)$ \\
\midrule

\LlamaOneBInst{} & 20 & 0.419 & 0.653 & -0.262 \\
 & 50 & 0.419 & 0.662 & -0.277 \\
 & 100 & 0.419 & \cellcolor{gray!12}0.665 & \cellcolor{gray!12}-0.283 \\
\addlinespace
\LlamaThreeBInst{} & 20 & 0.707 & 0.691 & -0.300 \\
 & 50 & 0.707 & 0.698 & -0.312 \\
 & 100 & 0.707 & \cellcolor{gray!12}0.700 & \cellcolor{gray!12}-0.315 \\
\addlinespace
\LlamaEightBInst{} & 20 & 0.749 & 0.684 & -0.276 \\
 & 50 & 0.749 & 0.689 & -0.284 \\
 & 100 & 0.749 & \cellcolor{gray!12}0.691 & \cellcolor{gray!12}-0.286 \\
\addlinespace
\QwenHalfBInst{} & 20 & 0.301 & \cellcolor{gray!12}0.717 & \cellcolor{gray!12}-0.344 \\
 & 50 & 0.301 & 0.716 & -0.343 \\
 & 100 & 0.301 & 0.715 & -0.341 \\
\addlinespace
\QwenTwoFiveInst{} & 20 & 0.416 & 0.639 & -0.237 \\
 & 50 & 0.416 & 0.641 & -0.241 \\
 & 100 & 0.416 & \cellcolor{gray!12}0.642 & \cellcolor{gray!12}-0.242 \\
\addlinespace
\QwenTwoFiveSevenBInst{} & 20 & 0.500 & 0.642 & -0.245 \\
 & 50 & 0.500 & 0.644 & -0.249 \\
 & 100 & 0.500 & \cellcolor{gray!12}0.644 & \cellcolor{gray!12}-0.249 \\
\addlinespace
\bottomrule
\end{tabular}
\end{adjustbox}
\label{tab:topk_multimodel}
\end{table*}

\section{Appendix: Early-Window Separability}
\label{sec:appendix-early-window}
To further control for length confounds, we compute $S_{50}=\max_{t\le 50}\It$, the max instability
within the first 50 steps, and evaluate separability using the same AUC$_{\text{wrong}}$ metric as
in \cref{tab:robustness}. The signal remains predictive under this early-window restriction
(\cref{tab:early_window_auc}).

\EarlyWarningFigureAppendix

\begin{table*}[t]
\caption{Early-window separability using $S_{50}=\max_{t\le 50}\It$ (GSM8K, 300 traces). AUC remains above
chance under the fixed early-window restriction.}%
\centering
\begin{adjustbox}{max width=\textwidth}
\begin{tabular}{l l c c}
\toprule
Model & Decoding & AUC$_{\text{wrong}}(S)$ & AUC$_{\text{wrong}}(S_{50})$ \\
\midrule
\QwenTwoFiveInst{} & greedy & 0.681 & 0.605 \\
\LlamaOneBInst{} & $\temp=0.0$ & 0.657 & 0.665 \\
\bottomrule
\end{tabular}
\end{adjustbox}
\label{tab:early_window_auc}
\end{table*}

\section{Appendix: Negative Controls and Alternative Baselines}
\label{sec:appendix-controls}
This appendix reports two reviewer-oriented controls on GSM8K (300 examples) using the same logged
traces as the main results for \LlamaFull{} greedy. First, we include
time-structure negative controls by randomizing the temporal order and recomputing strength
statistics. Second, we compare against simple entropy-family baselines computed from the same
per-step observables without rerunning the models.

\subsection{Time-Structure Negative Controls}
Because the peak-based strength statistic $S=\max_t \It$ is invariant to permutations of the
sequence values, time-structure controls are most informative for timing-aware or windowed
statistics (e.g., $S_w$ and peak position $\rho$), rather than for $S$ itself.
We consider two temporal randomizations. (1) \textbf{Shuffle $\{p_t\}$}: for each trace we
randomly permute the per-step logged top-$k$ distributions before recomputing $H_t$, $D_t$, and
$\It$. (2) \textbf{Shuffle $\{\It\}$}: for each trace we compute the original $\It$ sequence and
then randomly permute it. Since the main strength statistic $S=\max_t \It$ is invariant to
permutations of $\It$, this control primarily targets windowed statistics such as $S_{50}$.
All shuffles are deterministic per example identifier.

\begin{table*}[t]
\caption{Time-structure negative controls on GSM8K (300 examples) for greedy decoding.
Bucket slope is measured as the accuracy difference between bucket B5 and B1.
Shuffling $\{\It\}$ does not affect $S=\max_t \It$ by definition.}%
\centering
\small
\setlength{\tabcolsep}{4pt}
\begin{adjustbox}{max width=\textwidth}
\begin{tabular}{l l c c c}
\toprule
Model & Setting & \shortstack{AUC$_{\text{wrong}}$\\$(S)$} & \shortstack{Spearman\\$(S)$} & \shortstack{B5$-$B1\\$(S)$} \\
\midrule
\LlamaFull{} & original & 0.657 & $-0.178$ & $-0.183$ \\
\LlamaFull{} & shuffle $\{p_t\}$ & 0.652 & $-0.173$ & $-0.150$ \\
\LlamaFull{} & shuffle $\{\It\}$ & 0.657 & $-0.178$ & $-0.183$ \\
\bottomrule
\end{tabular}
\end{adjustbox}
\label{tab:time_shuffle_controls_s}
\end{table*}

\begin{table*}[t]
\caption{Time-structure negative controls on GSM8K (300 examples) for greedy decoding, evaluated with the
early-window statistic $S_{50}=\max_{t\le 50} \It$.
Shuffling reduces early-window separability, consistent with the role of temporal organization in
windowed diagnostics.}%
\centering
\small
\setlength{\tabcolsep}{4pt}
\begin{adjustbox}{max width=\textwidth}
\begin{tabular}{l l c c c}
\toprule
Model & Setting & \shortstack{AUC$_{\text{wrong}}$\\$(S_{50})$} & \shortstack{Spearman\\$(S_{50})$} & \shortstack{B5$-$B1\\$(S_{50})$} \\
\midrule
\LlamaFull{} & original & 0.664 & $-0.187$ & $-0.167$ \\
\LlamaFull{} & shuffle $\{p_t\}$ & 0.641 & $-0.161$ & $-0.133$ \\
\LlamaFull{} & shuffle $\{\It\}$ & 0.617 & $-0.134$ & $-0.117$ \\
\bottomrule
\end{tabular}
\end{adjustbox}
\label{tab:time_shuffle_controls_s50}
\end{table*}

\subsection{Entropy-Family Baselines}
We compare our main strength statistic $S_I=\max_t \It$ against three simple alternatives computed
from the same traces: $S_H=\max_t H_t$ (max entropy), $S_{\Delta H}=\max_t |H_t-H_{t-1}|$ (max entropy change),
and $S_D=\max_t D_t$ (max JSD). We also report the early-window variants obtained by restricting
the max to the first 50 steps.
We report both full-trace and fixed early-window variants below.

\EntropyFamilyBaselinesTable

\subsection{Signal Ablation (Additional)}
\LambdaAblationFigure

\begin{table*}[t]
\caption{Entropy-family baselines on GSM8K (300 examples) for greedy decoding, evaluated in the fixed
early window (first 50 steps).}%
\centering
\small
\setlength{\tabcolsep}{4pt}
\begin{adjustbox}{max width=\textwidth}
\begin{tabular}{l l c c}
\toprule
Model & Statistic & AUC$_{\text{wrong}}(S_{50})$ & Spearman$(S_{50})$ \\
\midrule
\LlamaFull{} & $S_{50,I}=\max_{t\le 50} \It$ & 0.664 & $-0.187$ \\
\LlamaFull{} & $S_{50,H}=\max_{t\le 50} H_t$ & 0.648 & $-0.169$ \\
\LlamaFull{} & $S_{50,\Delta H}=\max_{t\le 50} |H_t-H_{t-1}|$ & 0.544 & $-0.050$ \\
\LlamaFull{} & $S_{50,D}=\max_{t\le 50} D_t$ & 0.418 & $0.140$ \\
\bottomrule
\end{tabular}
\end{adjustbox}
\label{tab:entropy_family_baselines_s50}
\end{table*}

\section{Appendix: Bootstrap Confidence Intervals}
\label{sec:appendix-bootstrap}
To quantify uncertainty under the GSM8K 300-example setting, we compute percentile bootstrap
confidence intervals over examples (1000 resamples) for AUC$_{\text{wrong}}$ and bucketed accuracy.
\Cref{tab:bootstrap_auc_ci} reports AUC confidence intervals for $S$ and $S_{50}$, and
\Cref{fig:bootstrap_bucket_trends} visualizes bucket accuracy with 95\% confidence intervals.

\begin{table*}[t]
\caption{Bootstrap confidence intervals for AUC$_{\text{wrong}}$ on GSM8K (300 examples).
Intervals are percentile 95\% confidence intervals over 1000 bootstrap resamples.}%
\centering
\small
\setlength{\tabcolsep}{3pt}
\begin{adjustbox}{max width=\textwidth}
\begin{tabular}{l l c c}
\toprule
Model & Decoding & \shortstack{AUC$_{\text{wrong}}$\\$(S)$\\{[95\% CI]}} & \shortstack{AUC$_{\text{wrong}}$\\$(S_{50})$\\{[95\% CI]}} \\
\midrule
\LlamaFull{} & $\temp=0.0$ & 0.657 [0.553, 0.756] & 0.664 [0.545, 0.769] \\
\LlamaFull{} & $\temp=0.3$ & 0.667 [0.571, 0.757] & 0.679 [0.584, 0.772] \\
\LlamaFull{} & $\temp=0.7$ & 0.741 [0.636, 0.834] & 0.721 [0.614, 0.808] \\
\bottomrule
\end{tabular}
\end{adjustbox}
\label{tab:bootstrap_auc_ci}
\end{table*}

\begin{figure}[t]
\centering
\includegraphics[width=\linewidth]{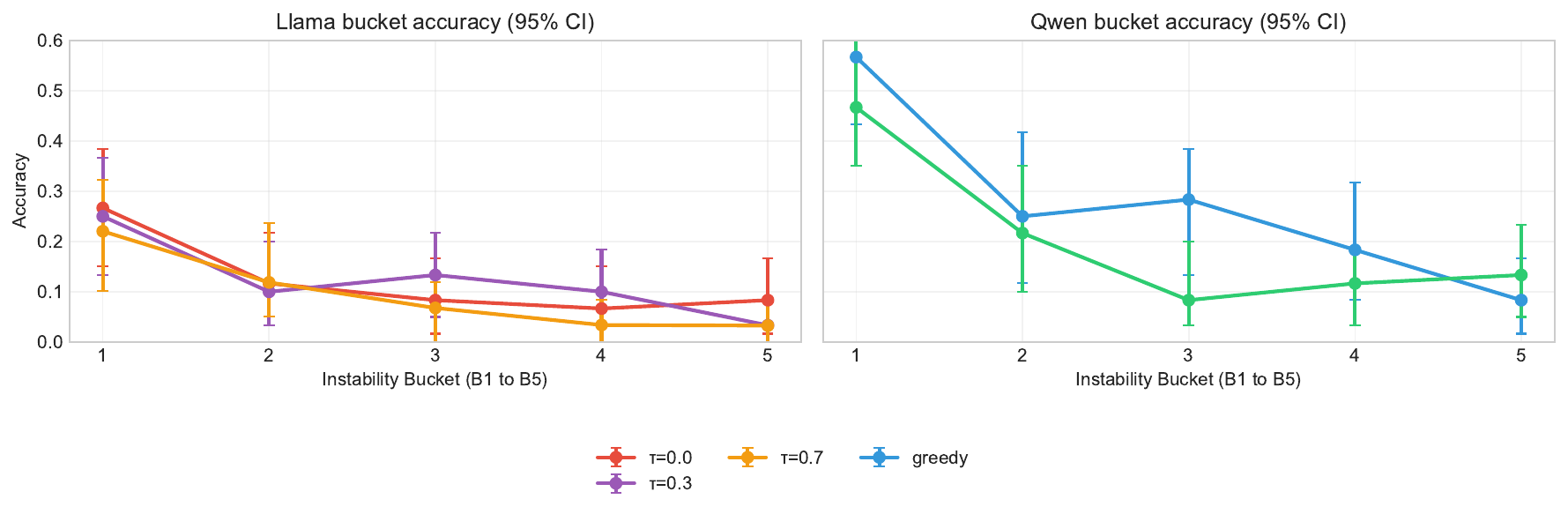}
\caption{Bootstrap confidence intervals for bucketed accuracy on GSM8K (300 examples).
Each curve plots accuracy across five equal-sized instability buckets (B1 to B5), with 95\%
bootstrap confidence intervals over examples (1000 resamples).}%
\label{fig:bootstrap_bucket_trends}
\end{figure}

\section{Appendix: Peak Threshold Sensitivity}
\label{sec:appendix-rho-thresholds}
We evaluate the robustness of the peak-timing interpretation by sweeping early and late thresholds
for the peak position $\rho=t^\star/T$ on the 100-trace full-vocabulary baseline run used in
\Cref{tab:corrective-destructive}. \Cref{tab:rho_threshold_sweep} shows that early-peak traces
consistently achieve higher accuracy than late-peak traces across a small threshold grid. We also
plot accuracy as a function of peak position in 10 equal-width bins (\cref{fig:rho_accuracy_bins}).

\begin{table*}[t]
\caption{Peak-timing threshold sweep on the 100-trace full-vocabulary baseline run.
Gap denotes Acc$_{\text{early}}$ minus Acc$_{\text{late}}$.}%
\centering
\small
\setlength{\tabcolsep}{3pt}
\begin{adjustbox}{max width=\textwidth}
\begin{tabular}{@{}c c r c r c c@{}}
\toprule
Early $\rho$ & Late $\rho$ & $N_{\text{early}}$ & Acc$_{\text{early}}$ & $N_{\text{late}}$ & Acc$_{\text{late}}$ & Gap \\
\midrule
0.20 & 0.45 & 39 & 0.538 & 19 & 0.211 & 0.328 \\
0.20 & 0.50 & 39 & 0.538 & 14 & 0.143 & 0.396 \\
0.20 & 0.60 & 39 & 0.538 & 8 & 0.125 & 0.413 \\
\midrule
0.25 & 0.45 & 57 & 0.456 & 19 & 0.211 & 0.246 \\
0.25 & 0.50 & 57 & 0.456 & 14 & 0.143 & 0.313 \\
0.25 & 0.60 & 57 & 0.456 & 8 & 0.125 & 0.331 \\
\midrule
0.30 & 0.45 & 68 & 0.441 & 19 & 0.211 & 0.231 \\
0.30 & 0.50 & 68 & 0.441 & 14 & 0.143 & 0.298 \\
0.30 & 0.60 & 68 & 0.441 & 8 & 0.125 & 0.316 \\
\bottomrule
\end{tabular}
\end{adjustbox}
\label{tab:rho_threshold_sweep}
\end{table*}

\begin{figure}[t]
\centering
\includegraphics[width=0.85\linewidth]{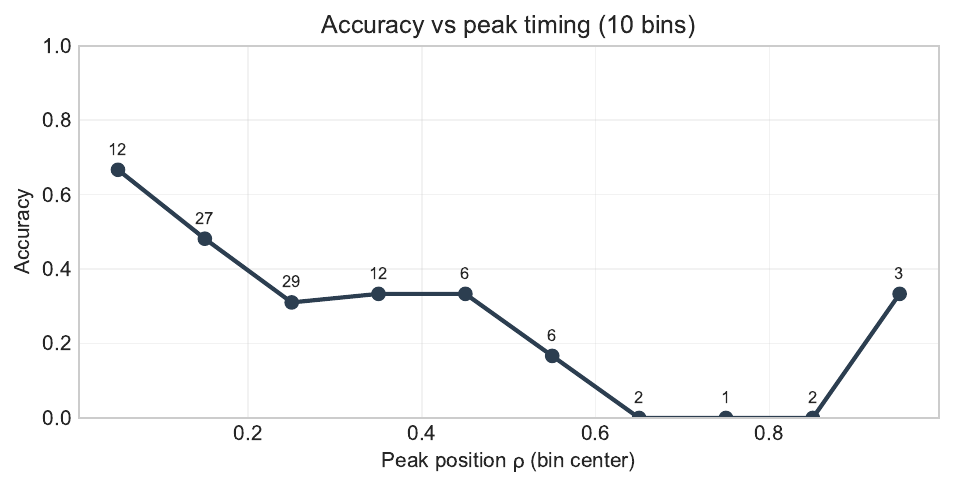}
\caption{Accuracy by peak position $\rho=t^\star/T$ in 10 equal-width bins on the 100-trace
full-vocabulary baseline run. Bin counts are annotated above each point.}%
\label{fig:rho_accuracy_bins}
\end{figure}

\section{Appendix: Mechanistic Interpretation}
\label{sec:appendix-mechanistic}
The discussion in the main text provides a brief, non-committal mechanistic interpretation of the
signal. Here we provide a longer interpretation, emphasizing that these abstractions do not imply
a unique internal mechanism.

\paragraph{Closed-loop amplification.}
Autoregressive decoding is a feedback system: the emitted token becomes part of the next input,
updating the residual stream and attention context, and therefore the next-token distribution. In
this view, regime shifts correspond to discrete reconfigurations of the effective computation
performed by attention and MLP blocks, which can amplify local sensitivity in the state update map.

\paragraph{Decision fragility and near ties.}
Our instability signal combines realized distributional change with uncertainty. When top candidates
are close in score, small state changes can flip the argmax. In a Transformer, such near ties can
arise when multiple partial solutions remain compatible with the prefix, leading to competing
attention patterns and feature activations. In this regime, entropy is a proxy for proximity to a
decision boundary, and it complements JSD when divergence estimates are compressed under top-$k$
logging.

\paragraph{Route switching and local curvature.}
Large distributional shifts can reflect route switching, where attention reallocates from one subset
of context tokens to another as the model commits to an intermediate value or revises an earlier
step. Prior work suggests that Transformers can implement discrete, circuit-like behaviors in
context~\cite{olsson2022incontext_induction}, which is compatible with abrupt changes in the
effective computation and observed JSD spikes.

\paragraph{Basins and stabilization time.}
The basin sets $\mathcal{A}_G$ and $\mathcal{A}_B$ and the stabilization time $\taumix$ provide a compact
way to state the recoverability intuition behind corrective versus destructive instability. In a
Transformer, these basins can be interpreted as regions of residual-stream activation space that
reliably produce continuations consistent with a correct or incorrect solution. The stabilization
time can be interpreted as the number of subsequent steps required for the model to propagate a
corrected intermediate value through the remaining reasoning and final-answer formatting, which
explains why late instability peaks are less recoverable.

\section{Appendix: Peak-Step Probes}
\label{sec:appendix-peak-probes}
To probe differences at the peak step $t^\star=\arg\max_t \It$, we analyze two additional metrics:
margin drops and support-set turnover.

\paragraph{(Log-)probability margin drop.}
Let $\Delta_t = \log \tilde{p}_t^{(1)} - \log \tilde{p}_t^{(2)}$ be the gap between the top-2
\emph{log probabilities} at step $t$ under the renormalized top-$k$ distribution $\tilde{p}_t$.
Here $\tilde{p}_t^{(1)}$ and $\tilde{p}_t^{(2)}$ denote the largest and second-largest probabilities under $\tilde{p}_t$ (ranked values), not powers.
We measure the \emph{margin drop} $\Delta_{t^\star-1} - \Delta_{t^\star}$ as a proxy for how abruptly
the model transitions from a locally dominant next-token preference to a more competitive state.
We use log probabilities under $\tilde{p}_t$ as a black-box proxy; analyzing raw logit margins
would require white-box access but is conceptually aligned with the same mechanistic interpretation.

\paragraph{Support-set turnover.}
Let $\mathcal{T}_t$ be the top-10 token set at step $t$. We compute the Jaccard \emph{overlap}
$J = |\mathcal{T}_{t^\star} \cap \mathcal{T}_{t^\star-1}| / |\mathcal{T}_{t^\star} \cup \mathcal{T}_{t^\star-1}|$,
and define turnover as $1-J$.
We use top-10 here to focus on the most competitive candidates at the peak step, while $\tilde{p}_t$
in our main signal uses top-$k$ (typically $k=50$) for a more stable entropy/JSD estimate.
High turnover implies low consecutive-step support overlap, which contributes directly to large
JSD spikes under truncated distributions.

\begin{table*}[t]
\caption{Peak-step characteristics for correct vs.\ wrong traces (\LlamaFull, GSM8K).
We aggregate 300 problems across three decoding settings (greedy; and stochastic sampling with
$\temp \in \{0.3,0.7\}$, top-$p$ 0.9), yielding 900 traces.
Correct traces show larger margin drops (decisive transition); wrong traces show higher turnover
(chaotic support-set change). (For per-setting robustness across this grid, see
\Cref{tab:robustness}; here we pool settings for mechanistic characterization.)}%
\centering
\small
\setlength{\tabcolsep}{3pt}
\begin{adjustbox}{max width=\textwidth}
\begin{tabular}{@{}l c c c c@{}}
\toprule
& \shortstack{Margin\\at $t^\star$} & \shortstack{Margin\\drop} & \shortstack{Jaccard\\overlap} & \shortstack{Turnover\\$(1-J)$} \\
\midrule
Correct & 0.30 & \textbf{1.82} & 0.18 & 0.82 \\
Wrong   & 0.45 & 1.05 & 0.16 & \textbf{0.84} \\
\bottomrule
\end{tabular}
\end{adjustbox}
\label{tab:spike-characteristics}
\end{table*}

\Cref{tab:spike-characteristics} reveals a key distinction: \textbf{correct traces exhibit
larger margin drops} (1.82 vs.\ 1.05) at the peak step $t^\star$, while \textbf{wrong traces
exhibit higher support-set turnover} (84\% vs.\ 82\%).
We treat these metrics as descriptive characterizations; some effect sizes (e.g., turnover) are modest,
but the patterns are consistent under pooling across decoding settings.

\section{Appendix: Additional Full-Set Runs}
\label{sec:appendix-fullset}
To complement the main analyses, we report dense per-step baseline runs for six models on the
full GSM8K test set (1319 examples) and on the HotpotQA distractor validation split (7405
examples). For each example we compute
$S=\max_t \It$ from the logged per-step instability and evaluate separability via AUC$_{\text{wrong}}$.
We also report $S_{50}=\max_{t\le 50}\It$ as an early-window control, and bucket accuracy by five
equal-sized quantile buckets of $S$. HotpotQA correctness uses a short-answer match on the first
line of the generated output after normalization, with a containment check against the reference.
The results of these full-set runs are summarized below.

\FullSetRunsAppendix

\section{Appendix: Peak Position in a Fixed Early Window}
\label{sec:appendix-fixed-window}
To control for time-budget confounds in the ``corrective vs.\ destructive'' analysis, we compute
peak position within a fixed early window (first 50 steps). Let
$t^\star_{50}=\arg\max_{t\le 50}\It$ (breaking ties by the smallest $t$) and $\rho_{50}=t^\star_{50}/50$ on the same held-out baseline run as
\Cref{tab:corrective-destructive}. The qualitative ordering remains: later peak positions
\emph{within the same window} correspond to lower accuracy (\cref{tab:fixed_window_peak}).

\begin{table*}[t]
\caption{Accuracy by peak position within a fixed early window (same 100-trace full vocabulary baseline as
\Cref{tab:corrective-destructive}). Even when peak position is defined within the first 50 steps,
later peaks are associated with lower accuracy.}%
\centering
\begin{adjustbox}{max width=\textwidth}
\begin{tabular}{l c c}
\toprule
Peak Position (first 50 steps) & \% of Traces & Accuracy \\
\midrule
Early ($\rho_{50} < 0.25$) & 13\% & 0.54 \\
Middle ($0.25 \le \rho_{50} \le 0.5$) & 34\% & 0.47 \\
Late ($\rho_{50} > 0.5$) & 53\% & 0.28 \\
\bottomrule
\end{tabular}
\end{adjustbox}
\label{tab:fixed_window_peak}
\end{table*}

\section{Appendix: Reproducibility and Implementation Details}
All results are generated via the project scripts in this repository with fixed configuration
files for dataset, model, decoding settings, and logging parameters. The pipeline consists of:
(i) generating traces by logging per-step top-$k$ token log probabilities during decoding,
(ii) computing $H_t$, $D_t$, and $\It$ from the renormalized top-$k$ distributions as described in
the Method section, and (iii) aggregating per-trace statistics ($S$ and early-window variants),
bucket summaries, and figures used in the paper. Unless otherwise noted, we use the GSM8K test
split with 300 problems, log top-$50$ tokens per-step, and cap generation at 128 new tokens.
The full-set runs use the GSM8K test split (1319 examples) and the HotpotQA distractor validation
split (7405 examples), evaluated with the same per-trace metrics and bucket procedure.
For stochastic decoding we use temperature $\temp=0.7$ with nucleus sampling ($p=0.9$) and a fixed
random seed; for \LlamaFull{} we additionally run $\temp=0.3$ under the same $p$ and seed. We also
include a small full vocabulary validation run for key timing claims (\cref{sec:corrective-destructive}).

\end{document}